\documentclass[letterpaper,11pt,reqno]{amsart}

\makeatletter
\usepackage{amssymb}
\usepackage{latexsym}
\usepackage{amsbsy}
\usepackage{amsfonts}
\usepackage{hyperref}
\usepackage{graphicx}
\usepackage{enumerate}
\usepackage{enumitem}
\usepackage{subcaption}
\usepackage{mathtools}
\usepackage{color}
\usepackage{hhline}

\usepackage{tikz}

\def\marginpar#1{\ignorespaces}

\DeclareMathOperator\tr{Tr}

\DeclareMathOperator\var{Var}
\DeclareMathOperator\argmin{\arg \min}

\newtheorem{theorem}{Theorem}[section]

\newtheorem{corollary}[theorem]{Corollary}

\newtheorem{assump}[theorem]{Assumption}

\numberwithin{equation}{section}
\makeatother
\begin{document}
\title[Diffusion models]{Score-based diffusion models via stochastic differential equations}

\author[Wenpin Tang]{{Wenpin} Tang}
\address{Department of Industrial Engineering and Operations Research, Columbia University. 
} \email{wt2319@columbia.edu}

\author[Hanyang Zhao]{{Hanyang} Zhao}
\address{Department of Industrial Engineering and Operations Research, Columbia University. 
} \email{hz2684@columbia.edu}

\date{Feburary 12, 2024} 
\begin{abstract}
This is an expository article on the score-based diffusion models, 
with a particular focus on the formulation via stochastic differential equations (SDE).
After a gentle introduction, 
we discuss the two pillars in the diffusion modeling --
sampling and score matching, 
which encompass the SDE/ODE sampling, 
score matching efficiency,
the consistency models, 
and reinforcement learning.
Short proofs are given to illustrate the main idea of the stated results. 
The article is primarily a technical introduction to the field,
and practitioners may also find some analysis useful in designing new models or algorithms.
\end{abstract}

\maketitle
\textit{Key words:} Diffusion models, discretization, generative models, ordinary differential equations, 
reinforcement learning, sampling, score matching, stochastic differential equations, total variation,
Wasserstein distance.

\smallskip
\textit{AMS 2020 Mathematics Subject Classification:} 60J60, 62E17, 65C30, 68P01.

\smallskip
\begin{center}
\qquad \qquad \qquad \qquad \qquad \qquad  {\em What I cannot create, I do not understand.   -- Richard Feynman}
\end{center}


\section{Introduction}
\label{sc1}

\quad Diffusion models 
describe a family of generative models 
that 
genuinely create 
the desired target distribution from noise.
Inspired from energy-based modeling \cite{Sohl15},
\cite{Ho20, Song19, Song20}
formalized
the idea of diffusion models,
which underpin the recent success
in the text-to-image creators 
such as
DALL·E 2 \cite{Ramesh22} and Stable Diffusion \cite{Rombach22},
and the text-to-video generators Sora \cite{Sora}, Make-A-Video \cite{Singer22} and Veo \cite{Veo}.
Roughly speaking,
{\em score-based diffusion models} rely on 
a forward-backward procedure\footnote{There is another class of generative models, the {\em flow-based models} in which:
\begin{itemize}[itemsep = 3 pt]
\item
The forward process transform the signal to a prescribed noise by a deterministic flow:
$X_0 \to X_1 \to \cdots \to X_n \stackrel{d}{=} p_{\tiny \mbox{noise}}(\cdot)$.
The noise $p_{\tiny \mbox{noise}}(\cdot)$ is often called the {\em prior}.
\item
The backward process reconstructs the signal from noise:  $X_n \to X_{n-1} \to \cdots \to X_0 \stackrel{d}{\approx} p_{\tiny \mbox{data}}(\cdot)$.
\end{itemize}

Examples include normalizing flows \cite{Lip23, PN21}, rectified flows \cite{Esser24, Liu22}, diffusion bridges \cite{DV21}, 
and consistency models \cite{song2023improved, SDC23}. 
See Section \ref{sc6} for a discussion on the probability flow ODE.}:

\begin{itemize}[itemsep = 3 pt]
\item
Forward deconstruction:
starting from the target distribution $X_0 \sim p_{\tiny \mbox{data}}(\cdot)$,
the model gradually adds noise
to transform the signal into noise 
$X_0 \to X_1 \to \cdots \to X_n \stackrel{d}{\approx} p_{\tiny \mbox{noise}}(\cdot)$.
\item
Backward construction:
 start with $X_n \sim p_{\tiny \mbox{noise}}(\cdot)$,
 and reverse the forward process 
 to recover the signal from noise
 $X_n \to X_{n-1} \to \cdots \to X_0 \stackrel{d}{\approx} p_{\tiny \mbox{data}}(\cdot)$.
\end{itemize}
The forward deconstruction is straightforward. 
What's the key in the diffusion models
is the backward construction,
and the underlying question is
how to reverse the forward process.
The answer hinges on two pillars:
{\em time reversal} of (Markov) diffusion processes to set the form of the backward process (Section \ref{sc2}),
and {\em score matching} to learn this process (Section \ref{sc3}).

\begin{figure}[hbtp]
    \begin{subfigure}{1\textwidth}
        \centering
        \includegraphics[width=0.14\linewidth]{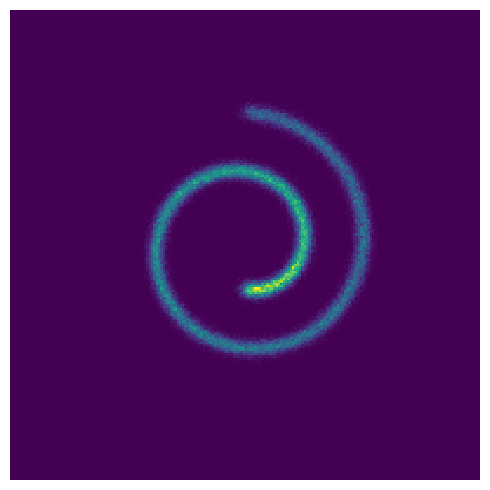}
        \caption{Original Swiss Roll}
        \label{fig:Swiss Roll Original}
    \end{subfigure}%
    
    \begin{subfigure}{1\textwidth}
        \centering
        \includegraphics[width=0.7\linewidth]{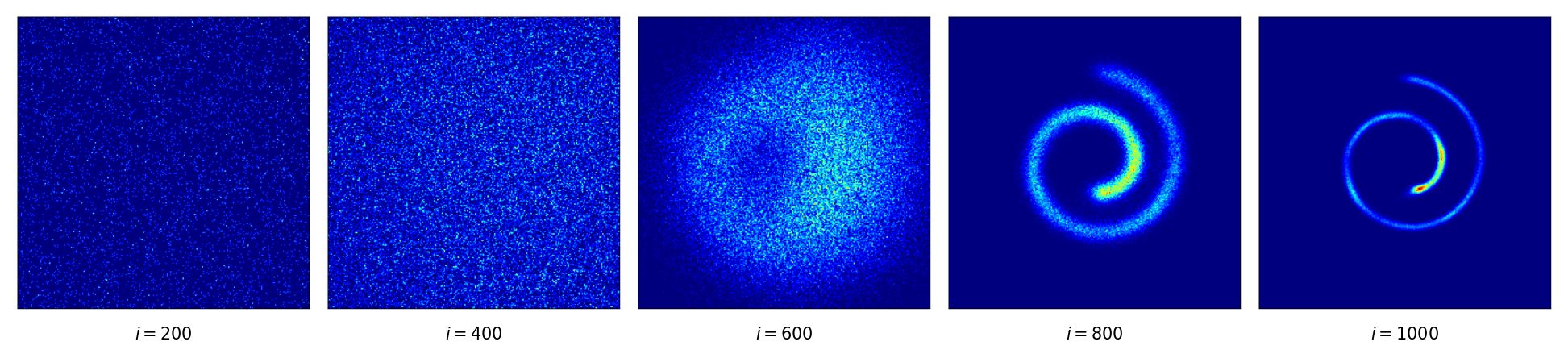}
        \caption{Swiss Roll generation with 200, 400, 600, 800, 10000 iterations.}
        \label{fig:Swiss Roll}
    \end{subfigure}
    \caption{Swiss Roll generation.}
\end{figure}

\quad In the original work \cite{Ho20, DDIM, Song19},
the forward/backward processes are specified by discrete-time Markov chains;
\cite{Song21, Song20} unified the previous models 
through the lens of stochastic differential equations (SDE).
In fact, there is no 
conceptual distinction between
discrete and continuous diffusion models
because 
the diffusion models specified by the SDEs can be regarded as 
the continuum limits of the discrete models (Section \ref{sc22}),
and the discrete diffusion models are obtained from 
the continuous models by suitable time discretization (Section \ref{sc42}).
The view is that 
the SDEs unveil structural properties of the models,
whereas the discrete counterparts give practical implementation.

\quad The purpose of this article is to provide a tutorial 
on the recent theory of score-based diffusion models,
mainly from a {\em continuous perspective} 
with a {\em statistical focus}. 
References on the discrete models will also be given. 
We sketch the proofs for most stated results,
and the assumptions are given only when
they are crucial in the analysis.
We often use the phrase ``under suitable conditions"
to avoid less important technical details,
and keep the presentation concise and to the point.
The paper serves as a gentle introduction to the field,
and practitioners may find
some analysis useful 
in designing new models or algorithms.
Since the SDE formulation is adopted,
we assume that the readers are familiar with
basic stochastic calculus.
\O ksendal's book \cite{Os98} provides a user-friendly account
for stochastic analysis,
and more advanced textbooks are \cite{KS91, SV79}.
See also \cite{CM242, Yang23} for a literature review and future directions in diffusion models,
and \cite{CM24, Win25} for more advanced materials such as
diffusion guidance and fine-tuning.

\quad The remainder of the paper is organized as follows. 
In Section \ref{sc2}, we start with the time reversal formula of diffusion processes,
which is the cornerstone of diffusion models.
Concrete examples are provided in Section \ref{sc22}.
Section \ref{sc3} is concerned with score matching techniques,
another key ingredient of diffusion models.
In Section \ref{sc6},
a deterministic sampler -- the probability ODE flow is introduced,
along with its application to the consistency models.
In Section \ref{sc4},
we consider the convergence of diffusion samplers.
Additional results on score matching are given in Section \ref{sc5}.
Concluding remarks and future directions 
are summarized in Section \ref{sc8}.

\medskip
\noindent
{\em Notations}: Below we collect a few notations that will be used throughout.
\begin{itemize}[itemsep = 3 pt]
\item
For $x, y$ vectors, denote by $x \cdot y$ the inner product between $x$ and $y$,
and $|x|$ the Euclidean ($L^2$) norm of $x$.
\item
For $A$ a matrix (or a vector), $A^\top$ denotes the transpose of $A$.
For $A$ a square matrix, $\tr(A)$ is the trace of $A$.
\item
For $f: \mathbb{R}^d \to \mathbb{R}$, $\nabla f$ is the gradient of $f$,  $\nabla \cdot f = div (f)$ is the divergence of $f$,
$\nabla^2 f$ is the Hessian of $f$,
and $\Delta f$ is the Laplacian of $f$.
For $f = (f_1, \ldots, f_n)$ with $f_i: \mathbb{R}^d \to \mathbb{R}$ its $i^{th}$ column,
$(\nabla f)^\top_i = \nabla f_i$ (the Jacobian matrix),
and 
$(\nabla \cdot f)_i = div(f_i)$ for $i = 1, \ldots, n$.
\item
The symbol $a = \mathcal{O}(b)$ means that $a/b$ is bounded,
and $a = o(b)$ means that $a/b$ tends to zero
as some problem parameter tends to $0$ or $\infty$.
\item
For $X$ a random variable, $\mathcal{L}(X)$ denotes the probability distribution of $X$,
and $X \sim p(\cdot)$ means that $X$ has the distribution $p(\cdot)$.
The notation $X \stackrel{d}{=} Y$ represents that $X$ and $Y$ have the same distribution.
\item
$\mathbb{E}X$ and $\var(X)$ denote the expectation and the variance of $X$ respectively.
We use $\widehat{\mathbb{E}}X$ for the empirical average of $X$, i.e. 
$\widehat{\mathbb{E}}X = \frac{1}{n} \sum_{i= 1}^n X_i$ for $X_1, \cdots, X_n$ independent and
identically distributed copies of $X$.
\item
$\mathcal{N}(\mu, \Sigma)$ denotes Gaussian distribution with mean $\mu$
and covariance matrix $\Sigma$.
\item
For $P$ and $Q$ two probability distributions, 
$d_{TV}(P,Q) = \sup_A|P(A) - Q(A)|$ is the total variation distance between $P$ and $Q$;
$\mbox{KL}(P,Q) = \int \log \left(\frac{dP}{dQ} \right)dP$ is the Kullback-Leilber (KL) divergence between $P$ and $Q$;
$W_2(P,Q) = \left(\inf_\gamma \mathbb{E}_{(X,Y) \sim \gamma} |X - Y|^2 \right)^{\frac{1}{2}}$,
where the infimum is taken over all couplings $\gamma$ of $P$ and $Q$,
is the Wasserstein-2 distance between $P$ and $Q$.
\end{itemize}

\section{Time reversal formula}
\label{sc2}

\quad In this section, 
we present the time reversal formula,
which lays the foundation for score-based diffusion models.

\quad Consider a general (forward) SDE:
\begin{equation}
\label{eq:SDE}
dX_t = f(t, X_t) dt + g(t,X_t) dB_t, \quad X_0 \sim p_{\tiny \mbox{data}}(\cdot),
\end{equation}
where $(B_t, \, t \ge 0)$ is $n$-dimensional Brownian motion\footnote{Write $B = (B^1, \ldots, B^n)$,
where $B^i$'s are independent standard Brownian motions.}, 
and $f: \mathbb{R}_+ \times \mathbb{R}^d \to \mathbb{R}^d$ 
and $g: \mathbb{R}_+ \times \mathbb{R}^d \to \mathbb{R}^{d \times n}$
are diffusion parameters.
Some conditions on $f(\cdot, \cdot)$ and $g(\cdot, \cdot)$ are required
so that the SDE \eqref{eq:SDE} is well-defined.
For instance,
\begin{itemize}[itemsep = 3 pt]
\item
If $f$ and $g$ are Lipschitz and have linear growth in $x$ uniformly in $t$, then \eqref{eq:SDE} has a strong solution which is pathwise unique.
\item
If $f$ is bounded, and $g$ is bounded, continuous and strictly elliptic, then \eqref{eq:SDE} has a weak solution which is unique in distribution.
\end{itemize}
See \cite{KS91, Os98} for background on the well-posedness of SDEs, 
and \cite[Chapter 1]{CE05} for a review of related results.

\quad For ease of presentation,
we assume that $X_t$ has a (suitably smooth) 
probability density $p(t, \cdot)$.
Let $T > 0$ be fixed, and run the SDE \eqref{eq:SDE} until time $T$ to get $X_T \sim p(T, \cdot)$.
Now if we start with $p(T, \cdot)$ and run the process $X$ backward for time $T$,
then we can generate copies of $p(0, \cdot) = p_{\tiny \mbox{data}}(\cdot)$.
More precisely, consider the time reversal $Y_t: = X_{T-t}$ for $0 \le t \le T$.
Assuming that $Y$ also satisfies an SDE, 
we can run the backward process:
\begin{equation*}
dY_t = \bar{f}(t, Y_t) dt + \bar{g}(t, Y_t ) dB_t, \quad Y_0 \sim p(T, \cdot).
\end{equation*}
So we generate the desired $Y_T  \sim p_{\tiny \mbox{data}}(\cdot)$ at time $T$.

\quad As mentioned earlier, 
the high-level idea of diffusion models is to create
the target distribution from noise.
This means that 
the noise should not depend on the target distribution.
Thus, we replace the (backward) initialization $Y_0 \sim p(T, \cdot)$ 
with some noise $p_{\tiny \mbox{noise}}(\cdot)$:
\begin{equation}
\label{eq:SDErev}
dY_t = \bar{f}(t, Y_t) dt + \bar{g}(t, Y_t ) dB_t, \quad Y_0 \sim p_{\tiny \mbox{noise}}(\cdot),
\end{equation}
as an approximation. Two natural questions arise: 
\begin{enumerate}[itemsep = 3 pt]
\item
How do we choose the noise $p_{\tiny \mbox{noise}}(\cdot)$?
\item
What are the coefficients $\bar{f}(\cdot, \cdot)$ and $\bar{g}(\cdot, \cdot)$?
\end{enumerate}

For (1), the noise $p_{\tiny \mbox{noise}}(\cdot)$ should be easily sampled. It is
commonly derived by 
decoupling $X_0$ from 
the conditional distribution of $(X_T \,|\, X_0)$,
as we will explain with the examples in Section \ref{sc22}. 
It is expected that the closer the the distributions $p(T, \cdot)$ and $p_{\tiny \mbox{noise}}(\cdot)$ are,
the closer the distribution of $Y_T$ sampled from \eqref{eq:SDErev} is to $p_{\tiny \mbox{data}}(\cdot)$.
For (2), it relies on the following result on the time reversal of SDEs.
To simplify the notations, we write 
\begin{equation*}
a(t,x):= g(t,x) g(t,x)^\top.
\end{equation*}

\begin{theorem}[Time reversal formula]
\cite{Ander82, HP86}
\label{thm:SDErev}
Under suitable conditions on $f(\cdot, \cdot)$, $g(\cdot, \cdot)$ and $\{p(t, \cdot)\}_{0 \le t \le T}$, 
we have:
\begin{equation}
\label{eq:coefrev}
\overline{g}(t,x) = g(T-t, x), \quad
\overline{f}(t,x) = -f(T-t, x) + \frac{\nabla \cdot (p(T-t,x) a(T-t,x))}{p(T-t, x)}.
\end{equation}
\end{theorem}

\begin{proof}
We give a heuristic derivation of the formula \eqref{eq:coefrev},
which consists of identifying the Fokker–Planck equation of 
$\overline{p}(t,x) := p(T-t, x)$ (the probability density of the time reversal $Y$).
First, the infinitesimal generator of $X$ is 
$\mathcal{L}:= \frac{1}{2} \nabla \cdot a(t,x) \nabla +f_a \cdot \nabla$,
where $f_a:= f - \frac{1}{2} \nabla \cdot a$.
It is known that the density $p(t,x)$ satisfies the Fokker–Planck equation:
\begin{equation*}
\frac{\partial}{\partial t} p(t,x) = \mathcal{L}^*p(t,x),
\end{equation*}
where $\mathcal{L}^*:= \frac{1}{2} \nabla \cdot a(t,x) \nabla - \nabla \cdot f_a$ is the adjoint operator of $\mathcal{L}$.
Thus,
\begin{equation}
\label{eq:FPfdbar}
\frac{\partial}{\partial t} \overline{p}(t,x) = -\frac{1}{2} \nabla \cdot \left(a(T-t, x) \, \nabla \overline{p}(t,x) \right) + \nabla \cdot \left(f_a(T-t, x) \, \overline{p}(t,x) \right).
\end{equation}
On the other hand, we expect the generator of $Y$ to be 
$\overline{\mathcal{L}}:=  \frac{1}{2} \nabla \cdot \overline{a}(t,x) \nabla + \overline{f}_{\overline{a}} \cdot \nabla$.
The Fokker-Planck equation for $\overline{p}(t,x)$ is:
\begin{equation}
\label{eq:FPbdbar}
\frac{\partial}{\partial t} \overline{p}(t,x) = \frac{1}{2} \nabla \cdot \left(\overline{a}(t, x) \, \nabla \overline{p}(t,x) \right) - \nabla \cdot \left(\overline{f}_{\overline{a}}(t, x) \, \overline{p}(t,x) \right).
\end{equation}
Comparing \eqref{eq:FPfdbar} and \eqref{eq:FPbdbar},
we set $\overline{a}(t,x) = a(T-t, x)$ and then get
\begin{equation*}
\left(f_a(T-t,x) + \overline{f}_{\overline{a}}(t,x)\right) \overline{p}(t,x) = a(T-t,x) \, \nabla \overline{p}(t,x),
\end{equation*}
which can be rewritten as:
\begin{equation*}
\left(f(T-t,x) + \overline{f}(t,x)\right) \overline{p}(t,x)- \nabla \cdot a(T-t, x) \, \overline{p}(t,x) = a(T-t,x) \, \nabla \overline{p}(t,x).
\end{equation*}
This yields the desired result. 
\end{proof}

\quad Let's comment on Theorem \ref{thm:SDErev}.
\cite{HP86, MNS89} proved the result by assuming that $f(\cdot, \cdot)$ and $g(\cdot, \cdot)$ are globally Lipschitz, and 
the density $p(t,x)$ satisfies an a priori $H^1$ bound. 
The implicit condition on $p(t,x)$ is guaranteed if $\partial_t + \mathcal{L}$ is hypoelliptic, or $\nabla^2 a(t,x)$ is uniformly bounded.
These conditions were relaxed in \cite{Quastel02}, where only the boundedness of $\nabla a(t,x)$ in some $L^2$ norm is required. 
In another direction, \cite{Fo85, Fo86} used an entropy argument to prove the time reversal formula in the case of $n =d$ and $g(t,x) = \sigma I$. 
This approach was further developed by \cite{CCGL21} in connection with optimal transport theory. 

By Theorem \ref{thm:SDErev}, the backward process is:
\begin{equation}
\label{eq:SDErevexp}
\begin{aligned}
dY_t = (-f(T-t, Y_t) + a(T-t, Y_t) \nabla & \log p(T-t,  Y_t) + \nabla \cdot a(T-t, Y_t))dt  \\
&+ g(T-t, Y_t ) dB_t, \quad Y_0 \sim p_{\tiny \mbox{noise}}(\cdot).
 \end{aligned}
\end{equation}
Since $f(\cdot, \cdot)$ and $g(\cdot, \cdot)$ are chosen in advance, 
all but the term $\nabla \log p(T-t, Y_t)$ in \eqref{eq:SDErevexp} are avaliable. 
So in order to run the backward process \eqref{eq:SDErevexp}, 
we need to learn $\nabla \log p(t,x)$,
known as {\em Stein's score function}.
Recently developed score-based generative modeling consists of 
estimating $\nabla \log p(t,x)$ by function approximations, 
which will be discussed in Section \ref{sc3}.

\section{Examples}
\label{sc22}

\quad We have seen that a diffusion model is specified by the pair
$(f(\cdot, \cdot), g(\cdot, \cdot))$.
The design of $(f(\cdot, \cdot), g(\cdot, \cdot))$ is important
because it determines the quality of data generation.
There are two general rules of thumb:
easy learning from the forward process,
and good sampling from the backward process,
which will be clear in the next two sections.

\quad Now let's provide some examples of the diffusion model \eqref{eq:SDE}--\eqref{eq:SDErevexp}.
Most existing models take the form:
\begin{equation*}
\label{eq:simplify}
n = d \quad \mbox{and} \quad g (t,x) = g(t)I,
\end{equation*}
where $g(t) \in \mathbb{R}_+$.
That is, the model parameter $g(\cdot, \cdot)$ is only time-dependent,
rather than (time and) state-dependent. 
One important reason is that 
for the SDEs with state-dependent coefficient,
it is often not easy to decouple $X_0$ from the distribution of $(X_T \,|\, X_0)$.
Hence, it is not clear how to pick $p_{\tiny \mbox{noise}}(\cdot)$
as a proxy to $X_T \sim p(T, \cdot)$,
as is the case of geometric Brownian motion.

\medskip
\noindent
(a) 
{\em Ornstein-Ulenback (OU) process} \cite{DV21}:
\begin{equation}
\label{eq:OUfg}
f(t,x) = \theta (\mu - x)  \mbox{ with } \theta > 0, \, \mu \in \mathbb{R}^d
\quad \mbox{and} \quad
g(t) = \sigma > 0.
\end{equation}
The distribution of $(X_t \,|\, X_0 = x)$ is 
$p(t, \cdot; x) = \mathcal{N}(\mu + (x - \mu) e^{-\theta t}, \frac{\sigma^2}{2\theta} (1 - e^{-2 \theta t}) I)$.
An obvious candidate for the noise is
the stationary distribution of the OU process
$p_{\tiny \mbox{noise}}(\cdot) =  \mathcal{N}(\mu, \frac{\sigma^2}{2 \theta} I)$.
The backward process is specified to:
\begin{equation}
\label{eq:OUrev2}
\begin{aligned}
d Y_t = \left(\theta(Y_t - \mu) + \sigma^2 \nabla \log p(T-t, Y_t) \right) dt & +g \, dB_t, 
\,\, Y_0 \sim \mathcal{N}\left(\mu, \frac{\sigma^2}{2 \theta} I\right).
\end{aligned}
\end{equation}
More generally, consider the overdamped Langevin process:
\begin{equation}
f(t,x) = - \nabla U(x) \quad \mbox{and} \quad g(t) = \sigma > 0,
\end{equation}
with a suitable landscape $U: \mathbb{R}^d \to \mathbb{R}$.
The OU process corresponds to the choice of $U(x) = \frac{\theta}{2}|x - \mu|^2$.
The backward process is then:
\begin{equation}
\label{eq:OUrev}
d Y_t = \left(\nabla U(Y_t)+ \sigma^2 \nabla \log p(T-t, Y_t) \right) dt  + \sigma dB_t, 
\,\, Y_0 \propto \exp\left(- \frac{2 \, U(\cdot)}{\sigma^2}\right).
\end{equation}

\noindent
(b) {\em Variance exploding (VE) SDE} \cite{Song20}: 
This is the continuum limit of {\em score matching with Langevin dynamics} (SMLD) \cite{Song19}.
The idea of SMLD is to use $N$ noise scales $\sigma_0 < \sigma_1 < \cdots < \sigma_{N-1}$, 
and run the (forward) Markov chain:
\begin{equation*}
x_i = x_{i-1} + \sqrt{\sigma_i^2 - \sigma_{i-1}^2} z_{i-1}, \quad 1 \le i \le N,
\end{equation*}
where $z_0, \ldots, z_{N-1}$ are independent and identically distributed $\mathcal{N}(0, I)$.
By taking $\Delta t = \frac{T}{N} \to 0$,  $\sigma(i\Delta t) = \sigma_i$, $x_{i\Delta t} = x_i$ and $z_{i\Delta t} = z_i$,
we get 
$x_{t+\Delta t} = x_t+ \sqrt{\sigma^2(t+\Delta t) - \sigma^2(t)} \, z_t \approx x_t + \sqrt{\frac{d \sigma^2(t)}{dt} \Delta t} \, z_t$.
That is,
\begin{equation*}
dX_t = \sqrt{\frac{d \sigma^2(t)}{dt}} \,  dB_t, \quad 0 \le t \le T.
\end{equation*}
This implies that $f(t,x) = 0$
and $g(t) = \sqrt{\frac{d \sigma^2(t)}{dt}}$.
 The noise scales are typically set to be a geometric sequence
$\sigma(t) = \sigma_{\min} \left( \frac{\sigma_{\max}}{\sigma_{\min}} \right)^{\frac{t}{T}}$,
with $\sigma_{\min} \ll  \sigma_{\max}$.
Thus,
\begin{equation}
\label{eq:sigVE}
f(t,x) = 0 \quad \mbox{and} \quad
g(t) = \sigma_{\min} \left(\frac{\sigma_{\max}}{\sigma_{\min}} \right)^{\frac{t}{T}} \sqrt{\frac{2}{T} \log \frac{\sigma_{\max}}{\sigma_{\min}}}.
\end{equation}
Here $X_t = X_0 + \int_0^t g(s) dB_s$ is the Paley-Wiener integral. 
The distribution of $(X_t \,|\, X_0 = x)$ is:
\begin{equation}
\label{eq:VEp}
\begin{aligned}
p(t, \cdot; x) &= \mathcal{N}\left(x, \left(\int_0^t g^2(s) ds\right) I \right) = \mathcal{N}\left(x, \sigma^2_{\min}\left( \left(\frac{\sigma_{\max}}{\sigma_{\min}} \right)^{\frac{2t}{T}}- 1 \right) I \right).
\end{aligned}
\end{equation}
The name ``variance exploding" comes from the fact that $\var(X_0) \ll \var(X_T)$
because $\sigma_{\min} \ll \sigma_{\max}$.
Note that the forward process $X$ does not have a stationary distribution,
but we can decouple $x$ from $p(T, \cdot) = \mathcal{N}\left(x, (\sigma^2_{\max} - \sigma^2_{\min}) I\right)$ 
to get
\begin{equation*}
p_{\tiny \mbox{noise}}(\cdot) = \mathcal{N}(0,(\sigma^2_{\max} - \sigma^2_{\min}) I)).
\end{equation*}
The backward process is:
\begin{equation}
\label{eq:VErev2}
d Y_t  = g^2(T-t) ) \nabla \log p(T-t, Y_t)+  g(T-t) dB_t,  \,\, Y_0 \sim \mathcal{N}(0,(\sigma^2_{\max} - \sigma^2_{\min}) I)),
\end{equation}
where $g(t)$ is defined by \eqref{eq:sigVE}.

\quad As we will see in Section \ref{sc6}, 
it is more convenient to use 
$\sigma(t) = t$ and $g(t) = \sqrt{2t}$ 
(suggested by \cite{karras2022elucidating} as a base model), 
with $p_{\tiny \mbox{noise}}(\cdot) = \mathcal{N}(0, T^2 I)$ as an alternative parametrization.

\medskip
\noindent
(c) {\em Variance preserving (VP) SDE} \cite{Song20}:
This is the continuum limit of {\em denoising diffusion probabilistic models} (DDPM) \cite{Ho20}.
DDPM uses $N$ noise scales $\beta_1 < \beta_2 < \cdots < \beta_N$,
and runs the (forward) Markov chain:
\begin{equation*}
x_i = \sqrt{1- \beta_i} x_{i-1} + \sqrt{\beta_i} z_i, \quad 1 \le i \le N.
\end{equation*}
Similarly, by taking the limit $\Delta t = \frac{T}{N} \to 0$ with $\beta(i\Delta t) = N \beta_i/T$,
$x_{i\Delta t} = x_i$ and $z_{i\Delta t} = z_i$,
we get
$x_{t+\Delta t} = \sqrt{1 - \beta(t) \Delta t} \, x_t + \sqrt{\beta(t) \Delta t}\, z_t
\approx x_t - \frac{1}{2} \beta(t) x_t \Delta t + \sqrt{\beta(t) \Delta t} \, z_t$.
This leads to:
\begin{equation*}
dX_t = -\frac{1}{2} \beta(t) X_t \, dt + \sqrt{\beta (t)} \, dB_t, \quad 0 \le t \le T.
\end{equation*}
We have:
\begin{equation}
\label{eq:sigbVP}
f(t,x) = -\frac{1}{2}\beta(t) x \quad \mbox{and} \quad g(t) = \sqrt{\beta(t)}.
\end{equation}
The noise scales of DDPM are typically an arithmetic sequence:
\begin{equation}
\label{eq:beta}
\beta(t)= \beta_{\min} + \frac{t}{T}(\beta_{\max} - \beta_{\min}), \quad \mbox{with } \beta_{\min} \ll \beta_{\max},
\end{equation}
(and $\beta_{\min}$, $\beta_{\max}$ are scaled to the order $N/T$.)
By applying It\^o's formula to $e^{\frac{1}{2}\int_0^t \beta(s) ds} X_t$,
we get the distribution of $(X_t \,|\, X_0 = x)$:
\begin{equation}
\label{eq:VPp}
\begin{aligned}
p(t, \cdot;x) = 
 \mathcal{N} \left( e^{-\frac{1}{2} \int_0^t \beta(s) ds} x, (1 - e^{-\int_0^t \beta(s) ds}) I  \right).
\end{aligned}
\end{equation}
Thus, 
$p(T, \cdot;x) = \mathcal{N}(e^{-\frac{T}{4} (\beta_{\max} + \beta_{\min})} x, (1 - e^{-\frac{T}{2} (\beta_{\max} + \beta_{\min})}) I)$,
which is close to $\mathcal{N}(0, I)$ if $\beta_{\max}$, or $T$ is set to be large. 
This justifies the name ``variance preserving", and
we can set $p_{\tiny \mbox{noise}} = \mathcal{N}(0,I)$.
The backward process is:
\begin{equation}
\label{eq:VPrev}
\begin{aligned}
d Y_t = \bigg(\frac{1}{2} \beta(T-t) Y_t + \beta(T-t) \nabla \log p(T-t, Y_t)\bigg)  dt 
+ & \sqrt{\beta(T-t))} dB_t, \\
& Y_0 \sim \mathcal{N}(0, I),
\end{aligned}
\end{equation}
where $\beta(t)$ is defined by \eqref{eq:beta}.

\quad Note that the choice of $\beta(t) = \sigma^2$ leads to the OU process with $f(t,x) = -\frac{1}{2} \sigma^2 x$ and $g(t) = \sigma$. 
In fact, the VP SDE is also referred to as the OU process 
in the literature \cite{Chen23, LLT22, LLT23}.
Here we emphasize the difference between the OU process,
and the VP SDE with the (typical) choice \eqref{eq:beta} for $\beta(t)$.

\medskip
\noindent
(d) {\em Sub-variance preserving (subVP) SDE} \cite{Song20}:
\begin{equation}
\label{eq:sigbSVP}
f(t,x) = -\frac{1}{2} \beta(t) x \quad \mbox{and} \quad
g(t) = \sqrt{\beta(t) (1 - e^{-2 \int_0^t \beta(s) ds})},
\end{equation}
where $\beta(t)$ is defined by \eqref{eq:beta}.
The same reasoning as in (c) shows that
the distribution of $(X_t\,|\, X_0 = x)$ is:
\begin{equation}
\label{eq:SVPp}
\begin{aligned}
p(t, \cdot;x) & =  \mathcal{N} \left( e^{-\frac{1}{2} \int_0^t \beta(s) ds} x, (1 - e^{-\int_0^t \beta(s) ds})^2 I  \right) \\
& =  \mathcal{N}\left( e^{-\frac{t^2}{4T}(\beta_{\max} - \beta_{\min}) - \frac{t}{2} \beta_{\min}}  x,  
(1 -e^{-\frac{t^2}{2T}(\beta_{\max} - \beta_{\min}) - t \beta_{\min}} )^2 I \right).
\end{aligned}
\end{equation}
Note that $\var_{ \tiny \mbox{sub-VP}}(X_t) \le \var_{ \tiny \mbox{VP}}(X_t)$, 
hence the name ``sub-VP",
and $p_{\tiny \mbox{noise}}(\cdot) = \mathcal{N}(0,I)$.
Set 
$$\gamma(t):= e^{-2 \int_0^t \beta(s) ds} = e^{-\frac{t^2}{T}(\beta_{\max} - \beta_{\min}) - 2t \beta_{\min}},$$
so
$g(t) = \sqrt{\beta(t) (1 - \gamma(t))}$.
The backward process is:
\begin{equation}
\label{eq:SVPrev2}
\begin{aligned}
d Y_t & = \bigg( \frac{1}{2} \beta(T-t) Y_t  + \beta(T-t)(1 - \gamma(T-t)) \nabla \log p(T-t, Y_t)\bigg)  dt \\
& \qquad \qquad \qquad \qquad+ \sqrt{\beta(T-t) (1 - \gamma(T-t))} dB_t, \quad Y_0 \sim \mathcal{N}(0, I).
\end{aligned}
\end{equation}

\noindent
(e) {\em Contractive diffusion probabilistic models} (CDPM) \cite{TZ24}:
The idea of CDPM is to force contraction on the backward process
to narrow the score matching errors
(at the cost of possible noise approximation bias).
A practical criterion is:
\begin{equation}
\label{eq:rfpositive}
(x -x') \cdot (f(t,x) - f(t, x')) \ge r_f(t) |x-x|^2, \quad \mbox{with }  \inf _{t \in [0, T]}r_f(t) > 0.
\end{equation}
For instance, we consider {\em contractive variance preserving (CVP) SDE}:
\begin{equation}
\label{eq:defCVP}
f(t,x) = \frac{1}{2} \beta(t)x \quad \mbox{and} \quad g(t) = \sqrt{\beta(t)},
\end{equation}
where $\beta(t)$ is defined by \eqref{eq:beta}.
Similar to (c), we set 
$$p_{\tiny \mbox{noise}}(\cdot) = \mathcal{N}\left( 0, 
(e^{\frac{T}{2} (\beta_{\max} + \beta_{\min})}-1) I \right).$$
The backward process is:
\begin{equation}
\begin{aligned}
d Y_t  = & \left(-\frac{1}{2} \beta(T-t) Y_t  +\beta(T-t) \nabla \log p(T-t, Y_t)\right)  dt \\
& \qquad \qquad + \sqrt{\beta(T-t)} dB_t, \quad 
Y_0 \sim \mathcal{N}\left( 0, 
(e^{\frac{T}{2} (\beta_{\max} + \beta_{\min})}-1 ) I \right).
\end{aligned}
\end{equation}


\quad Among these examples, VE and VP SDEs are the most widely used models
(e.g., Stable Diffusion uses VP, and consistency models rely on VE.)
CDPM achieves slightly better performance on certain tasks,
but needs to fine-tune its hyperparameters to trade off 
noise approximation bias and score matching errors.
In each of the above models,
one can estimate the corresponding score function $\nabla \log p(t,x)$ for backward sampling --
this is inconvenient because score matching often requires large computational cost. 
Nevertheless, there is a reparametrization trick via the probability flow ODE
that allows to sample from different models with a single score function of VE.
This will be discussed in Section \ref{sc61}.

\section{Score matching techniques}
\label{sc3}

\quad In Section \ref{sc2},
we have seen that 
the main obstacle for backward sampling \eqref{eq:SDErevexp}
is the unknown score function $\nabla \log p(t,x)$.
This section reviews the recently developed 
score-based generative modeling \cite{Ho20, Song19, Song20},
whose goal is to estimate the score function $\nabla \log p (t,x)$
by a family of functions $\{s_\theta(t,x)\}_\theta$ (e.g., kernels and neural nets).
This technique is referred to as {\em score matching}.
With the (true) score function $\nabla \log p (t,x)$ 
being replaced with the score matching function $s_\theta(t,x)$,
the backward process \eqref{eq:SDErevexp} becomes:
\begin{equation}
\label{eq:SDErevtheta}
\begin{aligned}
dY_t = (-f(T-t, Y_t) +  a(T-t, Y_t) & s_\theta (T-t, Y_t) 
 + \nabla \cdot a(T-t, Y_t))dt \\
& + g(T-t, Y_t ) dB_t, \quad Y_0 \sim p_{\tiny \mbox{noise}}(\cdot).
 \end{aligned}
\end{equation}
In the case of $g(t,x) = g(t) I$, it simplifies to:
\begin{equation}
\label{eq:SDErevtheta2}
\begin{aligned}
dY_t = \left(-f(T-t, Y_t) + g^2(T-t) s_\theta (T-t, Y_t) \right)dt  + & g(T-t ) dB_t, 
\\ &Y_0 \sim p_{\tiny \mbox{noise}}(\cdot).
\end{aligned}
\end{equation}

\quad The plan of this section is as follows.
In Section \ref{sc31},
we present the general score matching technique. 
The two (popularly used) scalable score matching methods --
{\em sliced score matching} and {\em denoising score matching}
will be discussed in 
Sections \ref{sc32} and \ref{sc33}.

\subsection{Score matching}
\label{sc31}

Recall that $\{s_\theta(t,x)\}_\theta$ is a family of functions on $\mathbb{R}_+ \times \mathbb{R}^d$
parametrized by $\theta$,
which are used to approximate the score function $\nabla \log p(t,x)$.
Fix time $t$, 
the goal is to solve the stochastic optimization problem:
\begin{equation}
\label{eq:scoremat}
\min_\theta \mathcal{J}_{\text{ESM}}(\theta):=\mathbb{E}_{p(t, \cdot)} |s_\theta(t,X) - \nabla \log p(t,X)|^2,
\end{equation}
where $\mathbb{E}_{p(t, \cdot)}$ denotes the expectation taken over $X \sim p(t, \cdot)$.
But
the problem \eqref{eq:scoremat}, 
known as the {\em explicit score matching} (ESM),
is far-fetched 
because the score $\nabla \log p (t,X)$ on the right side is not available.

\quad Interestingly, this problem has been studied in the context of estimating statistical models with unknown normalizing constant.
(In fact, if $p(\cdot)$ is a Gibbs measure, then its score $\nabla \log p(\cdot)$ does not depend on the normalizing constant.)
The following result 
shows that the score matching problem \eqref{eq:scoremat}
can be recast into 
a feasible stochastic optimization with no $\nabla \log p(t, X)$-term,
referred to as the {\em implicit score matching} (ISM).

\begin{theorem}[Implicit score matching]
\cite{Hyv05}
\label{thm:scoremat}
Let 
\begin{equation}
\label{eq:equivscoremat}
\mathcal{J}_{\text{ISM}}(\theta):= \mathbb{E}_{p(t, \cdot)} \left[ |s_\theta(t,X)|^2  + 2 \, \nabla \cdot s_\theta (t,X)\right].
\end{equation}
Under suitable conditions on $s_\theta$, we have 
$\mathcal{J}_{\text{ISM}}(\theta) = \mathcal{J}_{\text{ESM}}(\theta) + C$ for some $C$ independent of $\theta$. 
Consequently, the minimum point of $\mathcal{J}_{\text{ISM}}$ and that of $\mathcal{J}_{\text{ESM}}$ coincide. 
\end{theorem}
\begin{proof}
We have 
\begin{equation*}
\begin{aligned}
\nabla_\theta \mathcal{J}_{\text{ISM}}(\theta) 
& = \nabla_\theta \mathbb{E}_{p(t, \cdot)} \left[|s_\theta(t,X)|^2\right] - 2 \mathbb{E}_{p(t, \cdot)} \left[\nabla_\theta s_\theta(t,X)^\top \nabla \log p(t,X)\right] \\
& =  \nabla_\theta \mathbb{E}_{p(t, \cdot)} \left[|s_\theta(t,X)|^2\right]  - 2 \int \nabla_\theta s_\theta(t,x)^\top \nabla p(t,x) dx \\
& =\nabla_\theta \mathbb{E}_{p(t, \cdot)} \left[|s_\theta(t,X)|^2\right]  - 2 \, \nabla_\theta \int  s_\theta(t,x)^\top \nabla p(t,x) dx  \\ 
& = \nabla_\theta \mathbb{E}_{p(t, \cdot)} \left[|s_\theta(t,X)|^2\right] + 2 \nabla_\theta \int \nabla \cdot s_\theta(t,x) \, p(t,x) dx \\
& = \nabla_\theta \mathbb{E}_{p(t, \cdot)}\left[|s_\theta(t,X)|^2 + 2\, \nabla \cdot s_\theta(t,X) \right] = \nabla_\theta \widetilde{\mathcal{J}}(\theta),
\end{aligned}
\end{equation*}
where we use the divergence theorem in the fourth equation. 
\end{proof}

\quad Clearly, the implicit score matching problem \eqref{eq:equivscoremat} can be solved by stochastic optimization tools,
e.g., stochastic gradient descent (SGD). 
Here the distribution $p(t, \cdot)$ is sampled by
first sampling a random data point from $p_{\tiny \mbox{data}}(\cdot)$,
followed by the conditional distribution of $(X_t \,|\, X_0)$.
This requires
the distribution of $(X_t \,|\,X_0)$ be easy to sample
-- the easy learning criterion for diffusion models.
Notably,
the distribution of $(X_t \,|\, X_0)$
is Gaussian for all the examples in Section \ref{sc22}.

\quad In practice, a time-weighted version of the problem \eqref{eq:scoremat} is considered:
\begin{equation}
\label{eq:weighted score matching}
\begin{aligned}
\min_\theta \tilde{\mathcal{J}}_{\text{ESM}}(\theta) & := \mathbb{E}_{t \in \mathcal{U}(0, T)}\mathbb{E}_{p(t, \cdot)} \left[\lambda(t)|s_\theta(t,X) - \nabla \log p(t,X)|^2\right] \\
& = \frac{1}{T} \int_0^T \mathbb{E}_{p(t, \cdot)} \left[\lambda(t)|s_\theta(t,X) - \nabla \log p(t,X)|^2\right]dt,
\end{aligned}
\end{equation}
where $\mathcal{U}(0, T)$ denotes the uniform distribution over $[0, T]$,
and $\lambda: \mathbb{R} \rightarrow \mathbb{R}_{+}$ is a weight function.
The corresponding implicit score matching problem is:
\begin{equation}
\label{eq:equivscorematb}
\min_\theta \tilde{\mathcal{J}}_{\text{ISM}}(\theta) = \mathbb{E}_{t \in \mathcal{U}(0, T)}\mathbb{E}_{p(t, \cdot)} 
\left[\lambda(t) ( |s_\theta(t,X)|^2  + 2 \, \nabla \cdot s_\theta (t,X))\right].
\end{equation}

\quad Note, however, that
the problem \eqref{eq:equivscoremat} or \eqref{eq:equivscorematb}
can still be computationally expensive when the dimension $d$ is large. 
For instance, 
if we use a neural net as the function class of $s_{\theta}(t,x)$,
we need to perform $d$ times backpropagation of all the parameters
to compute $\nabla \cdot s_\theta(t,x) = \tr(\nabla  s_\theta(t,x))$.
This means that 
the computation of the derivatives scales linearly with the dimension,
hence making the gradient descent methods 
inefficient to solve the score matching problem
with respect to high dimensional data.
Two alternative approaches are commonly used to deal with the scalability issue,
which will be the focus of the next subsections.

\subsection{Sliced score matching}
\label{sc32}
 
The burden of computation comes from the term $\nabla \cdot s_{\theta}(t,x)$.
One clever idea proposed by \cite{Song20sl} is to tackle this term by random projections.
It relies on the key observation:
\begin{equation}
\label{eq:randompro}
\nabla \cdot s_{\theta}(t,x)=\mathbb{E}_{v\sim \mathcal{N}(0,I)}\left[v^\top \nabla s_{\theta}(t,x)v\right],
\end{equation}
where $\nabla s_{\theta}(t,x) \in \mathbb{R}^{d \times d}$ is the Jacobian matrix of $s_{\theta}(t,x)$.
The (implicit) score matching problem \eqref{eq:equivscorematb}
can then be rewritten as:
\begin{equation}
\label{eq:equivscorematb_slicedSM}
\min_\theta \tilde{\mathcal{J}}_{\text{SSM}}(\theta) = \mathbb{E}_{t \in \mathcal{U}(0, T)}\mathbb{E}_{v \sim \mathcal{N}(0,I)}\mathbb{E}_{p(t, \cdot)} \left[\lambda(t)\left(|s_\theta(t,X)|^2  + 2 \, v^\top \nabla (v^\top s_{\theta}(t,x))\right)\right],
\end{equation}
which is referred to as
the {\em sliced score matching} (SSM).
Since $\tilde{\mathcal{J}}_{\text{SSM}}(\theta) = \tilde{\mathcal{J}}_{\text{ISM}}(\theta)$ for all $\theta$,
we have the following result.
\begin{theorem}[Sliced score matching]
\cite{Song20sl}
Under suitable conditions on $s_{\theta}$,
we have $\tilde{\mathcal{J}}_{\text{SSM}}(\theta) = \tilde{\mathcal{J}}_{\text{ESM}}(\theta) + C$
for some $C$ independent of $\theta$.
Consequently, the minimum point of $\tilde{\mathcal{J}}_{\text{SSM}}(\theta)$ and
that of $\tilde{\mathcal{J}}_{\text{ESM}}(\theta)$ coincide. 
\end{theorem}

\quad Note that for a single fixed $v$, it only requires one-time backpropagation 
because the term $v^\top s_{\theta}(t,x)$ can be regarded as 
adding a layer of the inner product between $v$ and $s_{\theta}(t,x)$.
To get the expectation $\mathbb{E}_{v \sim \mathcal{N}(0, I)}$ in \eqref{eq:equivscorematb_slicedSM}, 
we pick $m$ samples of $v_i \sim \mathcal{N}(0,I)$, $1 \le i \le m$, 
compute the objective for each $v_i$,
and then take the average. 
So it requires $m$ times backpropagation.
Typically, $m$ is set to be small $(m \ll d)$,
and empirical studies show that $m = 1$ is often good enough.

\subsection{Denoising score matching}
\label{sc33}

The second approach relies on 
conditioning $X_t$ on $X_0 \sim p_{data}(\cdot)$,
known as denoising score matching (DSM) \cite{Hyv05, Vi11}.
Let's go back to the ESM problem \eqref{eq:weighted score matching},
and it is equivalent to the following DSM problem:
\begin{equation}
\label{eq:tildeDSM}
\begin{aligned}
&\tilde{\mathcal{J}}_{\text{DSM}}(\theta) \\
&=\mathbb{E}_{t\sim\mathcal{U}(0, T)}\left\{\lambda(t) \mathbb{E}_{X_0\sim p_{data}(\cdot)} \mathbb{E}_{X_t \mid X_0}\left[\left|s_{\theta}(t,X_t)- \nabla \log p(t,X_t \,|\, X_0)\right|^2\right]\right\},
\end{aligned}
\end{equation}
where the gradient $\nabla \log p(t,X_t \,|\, X_0)$ is with respect to $X_t$.

\begin{theorem}[Denoising score matching]
\cite{Vi11}
\label{thm:DSM}
Under suitable conditions on $s_\theta$, we have 
$\tilde{\mathcal{J}}_{\text{DSM}}(\theta) = \tilde{\mathcal{J}}_{\text{ESM}}(\theta) + C$ for some $C$ independent of $\theta$. 
Consequently, 
the minimum point of $\tilde{\mathcal{J}}_{\text{DSM}}$ and that of $\tilde{\mathcal{J}}_{\text{ESM}}$ coincide. 
\end{theorem}
\begin{proof}
Let 
$\mathcal{J}_{\text{DSM}}(\theta):= \mathbb{E}_{X_0\sim p_{data}(\cdot)} \mathbb{E}_{X_t | X_0}\left[\left|s_{\theta}(t,X_t)-\nabla \log p(t,X_t \,|\, X_0)\right|^2\right]$.
It suffices to prove that $\mathcal{J}_{\text{DSM}}(\theta) = \mathcal{J}_{\text{ESM}}(\theta) + C$ for some $C$
independent of $\theta$.
Note that
\begin{equation*}
\begin{aligned}
\mathcal{J}_{\text{ESM}}(\theta) 
& = \mathbb{E}_{p(t, \cdot)} |s_\theta(t,X) - \nabla \log p(t,X)|^2\\
& = \mathbb{E}_{p(t, \cdot)} \left[|s_\theta(t,X)|^2 - 2 s_\theta(t,X)^{\top}\nabla \log p(t,X)+|\nabla \log p(t,X)|^2\right].
\end{aligned}
\end{equation*}
For the inner product, we can rewrite it as:
\begin{equation*}
\begin{aligned}
& \mathbb{E}_{p(t, \cdot)} \left[s_\theta(t,X)^{\top}\nabla \log p(t,X)\right] \\
& \qquad = \int_{x}s_\theta(t,x)^{\top}\nabla p(t,x) dx\\
& \qquad = \int_{x}
s_\theta(t,x)^{\top}\nabla 
\int_{x_0}p(0,x_0)p(t,x; x_0)dx_0 dx\\
& \qquad = \int_{x_0}\int_{x}
s_\theta(t,x)^{\top} 
p(0,x_0)\nabla p(t,x; x_0)dx dx_0 \\
& \qquad = \int_{x_0}p(0,x_0)\int_{x}
s_\theta(t,x)^{\top} 
p(t,x; x_0) \nabla \log p(t,; |x_0)dx dx_0\\
& \qquad = \mathbb{E}_{X_0\sim p_{data}(\cdot)} \mathbb{E}_{X_t | X_0} \left[s_{\theta}(t,X_t)^{\top}\nabla \log p(t,X_t \,|\, X_0)\right].
\end{aligned}
\end{equation*}
Combining with $\mathbb{E}_{p(t, \cdot)} |s_\theta(t,X)|^2=\mathbb{E}_{X_0 \sim p_{data}(\cdot)} \mathbb{E}_{X_t | X_0}|s_\theta(t,X)|^2$ 
concludes the proof.
\end{proof}

\quad The main takeaway of DSM is that the gradient of the log density at some corrupted point should
ideally move towards the clean sample. 
As mentioned earlier, the conditional distribution of $(X_t \,|\, X_0)$ is required to be simple,
e.g., Gaussian for all the examples in Section \ref{sc22}.
Here we set:
\begin{equation*}
(X_t \,|\, X_0) \sim \mathcal{N}(\mu_t (X_0), \sigma_t^2 I) \quad \mbox{for some }
\mu_t(\cdot) \mbox{ and } \sigma_t > 0.
\end{equation*}
For instance, 
$\mu_t(x) = x$ and $\sigma_t = \sigma_{\min} \sqrt{\left(\frac{\sigma_{\max}}{\sigma_{\min}} \right)^{\frac{2t}{T}}-1}$ for VE,
and $\mu(t,x) = e^{-\frac{1}{2} \int_0^t \beta(s)ds} x$ and $\sigma_t = \sqrt{1 -  e^{- \int_0^t \beta(s)ds}}$
for VP.
In this case, we can compute explicitly the conditional score:
\begin{equation}
\label{eq:Gcscore}
\nabla \log p(t,X_t \,|\, X_0)=\frac{\mu_{t}(X_0)-X_t}{\sigma_t^2}.
\end{equation}
The direction $\frac{1}{\sigma_t^2}(X_t-\mu_{t}(X_0))$
clearly facilitates moving to the clean sample,
and we want $s_{\theta}(t,x)$ to match the score \eqref{eq:Gcscore} as best it can.
The problem \eqref{eq:tildeDSM} reads as:
\begin{equation}
\label{eq:tildeDSMbis}
\begin{aligned}
&\tilde{\mathcal{J}}_{\text{DSM}}(\theta) \\
&=\mathbb{E}_{t\sim\mathcal{U}(0, T)}\left\{\lambda(t) \mathbb{E}_{X_0\sim p_{data}(\cdot)} \mathbb{E}_{X_t \mid X_0}\left[\left|s_{\theta}(t,X_t) + \frac{X_t - \mu_t(X_0)}{\sigma_t^2} \right|^2\right]\right\}.
\end{aligned}
\end{equation}
Alternatively, the problem \eqref{eq:tildeDSMbis} can be interpreted as the least squares of $X_0$ over $X_t$
by {\em Tweedie's formula} \cite{Efron11}.
In fact, 
\begin{equation}
\mathbb{E}\left(\mu_t(X_0)  \, | \, X_t \right) =  X_t + \sigma_t^2 \nabla \log p(t, X_t).
\end{equation}
So $\nabla \log p(t,x)$ minimizes the loss
$\mathbb{E}_{X_0\sim p_{data}(\cdot)} \mathbb{E}_{X_t \mid X_0} \left|h(t,X_t) + \frac{X_t - \mu_t(X_0)}{\sigma^2} \right|^2$
over all measurable function $h$.
For function approximations, we replace $h$ with $\{s_\theta(t,x)\}_\theta$.
See also \cite{TTZ25} for Tweedie's formula generalized to non-Gaussian diffusion models.

\quad Empirically,
it was observed \cite{Song20} that 
a good candidate for the weight function $\lambda(t)$ is:
\begin{equation}
\label{eq:lambdagood}
\lambda(t) \propto 1 / \mathbb{E}\left|\nabla \log p(t,X_t \,|\, X_0)\right|^2=\sigma_t^2.
\end{equation}
The choice \eqref{eq:lambdagood}
is related to an evidence lower bound (ELBO) of 
the KL divergence between the generated distribution and the true distribution, 
see \cite{huang2021variational, luo2022understanding, Song21} for discussions.
See also \cite{KG23} for other choices of $\lambda(t)$.
Injecting \eqref{eq:lambdagood} into \eqref{eq:tildeDSMbis}
yields the ultimate objective:
\begin{equation}
\label{eq:ultDSM}
\begin{aligned}
\tilde{\mathcal{J}}_{\text{DSM}}(\theta)&=\mathbb{E}_{t\sim\mathcal{U}(0, T)}\left\{\sigma_t^2 \mathbb{E}_{X_0\sim p_{data}(\cdot)} \mathbb{E}_{X_t \mid X_0}\left[\left|s_{\theta}(t,X_t) + \frac{X_t - \mu_t(X_0)}{\sigma_t^2}\right|^2\right]\right\}\\
& = \mathbb{E}_{t\sim\mathcal{U}(0, T)}\left\{ \mathbb{E}_{X_0\sim p_{data}(\cdot)} \mathbb{E}_{\epsilon \sim\mathcal{N}(0,I)}\left[\left|\sigma_t s_{\theta}(t,\mu_{t}(X_0)+\sigma_t\epsilon)+\epsilon \right|^2\right]\right\},
\end{aligned}
\end{equation}
where the second equality 
follows from a change of variables $\epsilon:= \frac{X_t - \mu_t(X_0)}{\sigma_t}$.

\quad Denoising score matching  \eqref{eq:ultDSM}, which is essentially a least squares regression, 
is the most widely used approach for score matching.
It often requires large batches to solve the stochastic optimization problem.
Sliced score matching \eqref{eq:equivscorematb_slicedSM} is computationally efficient in high dimensions.
However, its performance relies on the random projections, i.e., the choice of slices,
which may incur high variance when the number of samples $m$ is small.

\section{ODE samplers and consistency model}
\label{sc6}

\quad The stochastic sampler \eqref{eq:SDErevtheta2} (and
its discretization)
may suffer from slow convergence and instability due to the randomness.
Moreover, estimating the score function for each model is computationally expensive.
In this section, we consider sampling $Y_T$ by an equivalent 
deterministic scheme -- the ODE sampler.
Section \ref{sc61} presents the probability flow ODE 
that leads to the ODE sampler of diffusion models.
We also explain the reparametrization trick (mentioned at the end of Section \ref{sc22}),
which avoids score matching repeatedly for different models.
In Section \ref{sc62}, 
we discuss one prominent recent application of ODE sampling -- the consistency model.

\subsection{The probability flow ODE}
\label{sc61}

The main idea is that
starting from $Y'_0 \sim p(T, \cdot)$,
and going through an ODE rather than an SDE
will produce a sample $Y'_T \sim p_{\tiny \mbox{data}}(\cdot)$. 
Here we are only concerned with the sampled distribution at time $T$.

\quad The equivalent ODE formulation is a consequence of the following result,
known as the {\em probability flow ODE}.
\begin{theorem}[Probability flow ODE]
\label{thm:pfODE}
Let $(X_t, \, 0 \le t \le T)$ be defined by \eqref{eq:SDE},
with $p(t, \cdot)$ the probability distribution of $X_t$.
Let 
\begin{equation}
\widetilde{f}(t,x):= f(t,x) - \frac{1}{2} \nabla \cdot (g(t,x)g(t,x)^\top) - \frac{1}{2} g(t,x) g(t,x)^\top \nabla \log p (t,x),
\end{equation}
and let $(\widetilde{X}_t, \,  0\le t \le T)$ solve the ODE:
\begin{equation}
\label{eq:equivODE}
\frac{d\widetilde{X}_t}{dt} = \widetilde{f}(t, \widetilde{X}_t), \quad \widetilde{X}_0 \sim p_{\tiny \mbox{data}}(\cdot).
\end{equation}
Then for each $t$, $X_t$ and $\widetilde{X}_t$ have the same distribution. 
\end{theorem}
\begin{proof}
The idea is to match the Fokker-Planck equation in both settings. 
Denote by $\widetilde{p}(t,x)$ the probability distribution of $\widetilde{X}_t$ defined by \eqref{eq:equivODE}.
By the (first-order) Fokker-Planck equation, $\widetilde{p}(t,x)$ solves
\begin{equation}
\frac{\partial \widetilde{p}(t,x)}{\partial t}
= -\nabla \cdot \left(\widetilde{f}(t,x) \widetilde{p}(t,x) \right), \quad \widetilde{p}(0,x) = p_{\tiny \mbox{data}}(x).
\end{equation}
On the other hand, $p(t,x)$ solves the (second order) Fokker-Planck equation:
\begin{equation}
\label{eq:FP2nd}
\frac{\partial p(t,x)}{\partial t} = - \nabla \cdot \left( f(t,x) p(t,x)\right) + 
\frac{1}{2} \nabla^2 : \left(g(t,x) g(t,x)^\top p(t,x) \right), \, p(0, x) = p_{\tiny \mbox{data}}(x).
\end{equation}
where $:$ denotes the Frobenius inner product. 
By algebraic rearrangement, \eqref{eq:FP2nd} simplifies to:
\begin{equation}
\label{eq:FP2nd2}
\frac{\partial p(t,x)}{\partial t} = - \nabla \cdot \left( \widetilde{f}(t,x) p(t,x)\right),
\quad p(0, x) = p_{\tiny \mbox{init}}(x).
\end{equation}
Comparing \eqref{eq:equivODE} with \eqref{eq:FP2nd} yields $\widetilde{p}(t, \cdot) = p(t,\cdot)$.
\end{proof}

\quad We remark that the two processes $X$ and $\widetilde{X}$ only have the same marginal distributions,
but they don't have the same distribution at the level of the process (nor the joint distributions).

\quad Now let's apply the equivalent ODE formulation to the diffusion models. 
We specify to the practical case where $n = d$ and $\sigma(t,x) = \sigma(t) I$.
The corresponding ODE is given by
\begin{equation}
\label{eq:ODDE}
\frac{dX_t}{dt} = f(t,X_t) - \frac{1}{2} g^2(t) \nabla \log p (t, X_t), 
\quad X_0 \sim p_{\tiny \mbox{data}}(\cdot),
\end{equation}
and the backward ODE is then:
\begin{equation}
\label{eq:ODEreverse}
\frac{dY_t}{ dt} =  - f(T-t, Y_t) + \frac{1}{2} g^2(T-t)\nabla \log p(T-t, Y_t), \quad Y_0 \sim p_{\tiny \mbox{noise}}(\cdot).
\end{equation}
With the score function $\nabla \log p(t,x)$ being replaced with the score matching function $s_\theta(t,x)$,
the backward ODE becomes:
\begin{equation}
\label{eq:ODEreverse2}
\frac{dY_t}{ dt} = - f(T-t, Y_t) + \frac{1}{2} g^2(T-t)s_\theta(T-t, Y_t), \quad Y_0 \sim p_{\tiny \mbox{noise}}(\cdot),
\end{equation}
which can be regarded as the continuum limit of {\em denoising diffusion implicit models} (DDIM) \cite{DDIM}.
Among the examples in Section \ref{sc22}:


\smallskip
\noindent
(b') The ODE sampler for VE is:
\begin{equation}
\label{eq:ODEVE}
\frac{dY_t}{ dt} =   \frac{1}{2} g^2(T-t) s_\theta(T-t, Y_t), \quad Y_0 \sim p_{\tiny \mbox{noise}}(\cdot),
\end{equation}
where $g$ is defined by \eqref{eq:sigVE}. 
An alternative parametrization is $g(t) = \sqrt{2 t}$.

\medskip
\noindent
(c') 
The ODE sampler for VP is:
\begin{equation}
\label{eq:ODEVP}
\frac{dY_t}{ dt} =  \frac{1}{2} \beta(T-t) \left( Y_t + s_\theta(T-t, Y_t)\right), \quad Y_0 \sim p_{\tiny \mbox{noise}}(\cdot).
\end{equation}
where $\beta(t)$ is defined by \eqref{eq:beta}.



The ODEs \eqref{eq:ODEVE}--\eqref{eq:ODEVP}
can be easily solved by existing ODE solvers, see Section \ref{sc42} for further discussions.

\quad Now let's explain the reparametrization trick, which allows to sample from different models
with the score function of VE.
Let $p^{\tiny \mbox{VE}}(t,x)$ be the probability density of $X_t$, with $f(t,x) = 0$ and $g(t)= \sqrt{2t}$.
Note that in all the above examples, 
$f(\cdot, \cdot)$ takes the form $f(t,x) = f(t) x$ with $f(t) \in \mathbb{R}$.
In this case, by a change of variables \cite{karras2022elucidating},
the diffusion model can be reparametrized by
\begin{equation}
    s(t):= \exp\left(\int_0^t f(r) dr \right) \quad \mbox{and} \quad \ell(t):=\sqrt{\int_0^t \frac{g(r)^2}{s(r)^2} dr},
\end{equation}
 so that the ODE \eqref{eq:ODDE} becomes:
\begin{equation}
    \frac{dX_t}{dt} = \frac{s'(t)}{s(t)} X_t - s(t)^2 \ell'(t) \ell(t) \nabla \log p^{\tiny \mbox{VE}}\left(\ell(t), \frac{X_t}{s(t)} \right), \,\, X_0 \sim p_{\tiny \mbox{data}}(\cdot).
\end{equation}
That is, a diffusion model with $f(t,x) = f(t) x$ can be represented in terms of VE by a suitable space-time change.
For instance, 
$s(t) = 1$ and $\ell(t) = t$ for VE with $f(t) = 0$ and $g(t) = \sqrt{2t}$,
and 
$s(t) =e^{-\frac{1}{2}\int_0^t \beta(r) dr}$ and $\ell (t) = \sqrt{e^{\int_0^t \beta(r) dr} - 1}$ for VP with $f(t) = -\frac{1}{2} \beta (t)$ and $g(t) = \sqrt{\beta(t)}$.

\quad Using the score matching techniques in Section \ref{sc3}, 
we can approximate the score $\nabla \log p^{\tiny \mbox{VE}}(t,x)$ by a class of functions $\{s_\theta^{\tiny \mbox{VE}}(t,x)\}_{\theta}$.
The ODE sampler is then:
\begin{equation}
\begin{aligned}
\frac{dY_t}{dt} = -\frac{s'(T-t)}{s(T-t)} Y_t + s(T-t)^2 \ell'(T-t) \ell(T-t) & s_\theta^{\tiny \mbox{VE}}\left( \ell(T-t), \frac{Y_t}{s(T-t)}\right),\\
& \qquad \qquad Y_0 \sim p_{\tiny \mbox{noise}}(\cdot).
\end{aligned}
\end{equation}
In particular, the ODE sampler for VP is:
\begin{equation*}
\begin{aligned}
\frac{dY_t}{dt} = \frac{1}{2} \beta(T-t)  \bigg(Y_t + s_\theta^{\tiny \mbox{VE}} \bigg(\sqrt{e^{\int_0^{T-t} \beta(r) dr} - 1},  Y_t & e^{\frac{1}{2}\int_0^{T-t} \beta(r) dr}\bigg) \bigg), \\
& \quad \,\,\,  Y_0 \sim p_{\tiny \mbox{noise}}(\cdot).
\end{aligned}
\end{equation*}

\subsection{Consistency model}
\label{sc62}

Here we present an application of the probability flow ODE --
the consistency model \cite{SDC23} 
(see \cite{Kim23, song2023improved} for a more advanced version,
and \cite{LCM} for its latest application.)
The advantage of the consistency modeling 
is that it only performs a {\em one-step} sampling
-- 
this is in contrast with the SDE or the ODE sampling,
where multiple steps are required in discretization.

\quad The main takeaway
of the consistency model is to approximate the ODE flow. 
Recall the ODE sampler from \eqref{eq:ODEreverse}. 
 The ODE theory implies that 
there is a function (the ODE flow) 
$F: \mathbb{R}_{+} \times \mathbb{R}_{+} \times \mathbb{R}^d \to \mathbb{R}^d$ such that 
\begin{equation}
\label{eq:groupprop}
Y_t = F(t,s,Y_s), \quad \mbox{for } 0 \le s< t \le T.
\end{equation}
Thus, $Y_T = F(T, t, Y_t)$ for any $t$.
The consistency model proposes to learn 
$F(T, \cdot, \cdot)$ directly.
To be more precise,
the function of form
$F_\theta(T-t, y)$
is used to approximate the flow $F(T, t, y)$.
(Note that if the ODE \eqref{eq:ODEreverse} is autonomous, e.g. the OU process,
the flow $F(t,s, y)$ is indeed of form $F(t-s, y)$.)
With the learned flow $F_\theta$,
the model outputs 
the sample $F_\theta(T, Y_0)$,
$Y_0 \sim p_{\tiny \mbox{noise}}(\cdot)$
in just one step.

\quad Now the question is how to learn the ODE flow $F(T,t, y) \approx F_\theta(T-t, y)$.
Two methods were proposed based on
the group property \eqref{eq:groupprop}. 

\medskip
\noindent
(a) {\em Consistency distillation} (CD): 
the approach builds on top of a pretrained score function $s_\theta(t,x)$.
Take $t \sim \mathcal{U}(0,T)$.
First sample $Y_+ \stackrel{d}{=}  X_t$ of the forward process
(VE with $g(t) = \sqrt{2t}$ is used, so $X_t \sim \mathcal{N}(x, t^2 I$) with $x \sim p_{\tiny \mbox{data}}(\cdot)$).
Then use the ODE sampler \eqref{eq:ODEreverse} to set:
\begin{equation}
Y_{-} = Y_+ + \left(-f(t, Y_+) + \frac{1}{2} g^2(t) s_\theta(t, Y_+)\right) \delta,
\end{equation}
for some small $\delta$.
Roughly speaking, 
$(Y_+, Y_{-})$ are on the ODE flow $F$ at the times $T-t$ and $T- t + \delta$.
Inspired by the fact that
\begin{equation*}
\underbrace{F(T, T-t, Y_{T-t})}_{\approx F_\theta(t, Y_+)} = Y_T = \underbrace{F(T, T-t+ \delta, Y_{T-t + \delta})}_{\approx F_\theta(t - \delta, Y_{-})},
\end{equation*}
the training objective is:
\begin{equation}
\label{eq:CDTtrain}
\mathbb{E}_{t \in \mathcal{U}(0,T)} \mathbb{E}[\lambda(t) |F_\theta(t, Y_+) - F_\theta(t - \delta, Y_-)|^2],
\end{equation}
where $\lambda(t)$ is a weight function. 

\medskip
\noindent
(b) {\em Consistency training} (CT): 
the approach is similar to CD, except that
$Y_- \stackrel{d}{=} X_{t - \delta}$ is taken {\em independently} 
of $Y_{+} \stackrel{d}{=} X_t$.
(In the VE setting with $g(t) = \sqrt{2t}$, $Y_{+} \sim \mathcal{N}(x, t^2 I)$ and 
$Y_{-} \sim \mathcal{N}(x, (t- \delta)^2 I)$ independently.)
Thus, CT does not require a pretrained score function. 
Note that 
\begin{equation*}
Y_{+} \stackrel{d}{=} Y_{T-t} \quad \mbox{and} \quad Y_{-} \stackrel{d}{=} Y_{T-t+ \delta} \quad \mbox{marginally}.  
\end{equation*}
But in contrast with CD,
$(Y_{+}, Y_{-})$ are not on the ODE flow path, i.e., a wrong coupling.

\quad Now we consider a simple scenario,
where $g(t) = \sqrt{2t}$ and $p_{\tiny \mbox{data}}(\cdot) = \delta_0$ (delta mass at $0$). 
It is easy to compute that
\begin{equation*}
(Y^{CT}_{+}, Y^{CT}_{-}) \stackrel{d}{=} \mathcal{N}\left(0, \begin{pmatrix}
t^2 I & 0 \\
0 & (t - \delta)^2 I
\end{pmatrix}
  \right), 
 \end{equation*}
 \begin{equation*}
 (Y^{CD}_{+}, Y^{CD}_{-}) \stackrel{d}{=} \mathcal{N}\left(0, \begin{pmatrix}
t^2 I & t(t - \delta) I \\
t(t - \delta) I & (t - \delta)^2 I
\end{pmatrix}
 \right).
 \end{equation*}
As a result,
\begin{equation}
\begin{aligned}
W^2_2((Y^{CD}_{+}, Y^{CD}_{-}), (Y^{CT}_{+}, Y^{CT}_{-}))
& = 2d \left(t^2 + (t - \delta)^2 - \sqrt{t^4 + (t - \delta)^4} \right) \\
& \approx  2(2 - \sqrt{2})  t^2 d.
\end{aligned}
\end{equation}
So the distance between $(Y^{CD}_{+}, Y^{CD}_{-})$ 
and $(Y^{CT}_{+}, Y^{CT}_{-})$
is far from being negligible. 
By taking $t \sim \mathcal{U}(0,1)$,
the expected $W_2$ distance is of order
$\sqrt{\left(1 - \frac{\sqrt{2}}{2}\right)d}$.
This explains why CD outperforms CT empirically \cite{SDC23}.
A remedy to CT has recently been proposed in \cite{song2023improved}.

\quad Also note that the objective function \eqref{eq:CDTtrain}
can be lifted to continuous time via partial differential relation
of the ODE flow $Y_T = F(T, T-t, Y_{T-t}) \approx F_\theta(t, Y_{T-t}) $:
\begin{equation}
\label{eq:CDTtrain1}
\mathbb{E}_{t \in \mathcal{U}(0,T)} \mathbb{E}\left[\lambda(t) \left|\partial_t F_\theta(t, Y_{+}) + \nabla F_{\theta}(t, Y_+)\left(f(t, Y_+) - \frac{1}{2} g^2(t) s_\theta(t, Y_+)\right) \right|^2\right],
\end{equation}
or
\begin{equation}
\label{eq:CDTtrain2}
\mathbb{E}_{t \in \mathcal{U}(0,T)} \mathbb{E}\left[\lambda(t) F_\theta(t, Y_+) \cdot \left(\partial_t F_\theta(t, Y_{+}) + \nabla F_{\theta}(t, Y_+) \frac{Y_+ - x}{t} \right)\right],
\end{equation}
where $Y_+ \stackrel{d}{=} X_t$ and $x \sim p_{\tiny \mbox{data}}(\cdot)$.
The functions \eqref{eq:CDTtrain1} and \eqref{eq:CDTtrain2}
are the continuum limits 
of CD and CT respectively
after suitable scaling. 

\section{Convergence results}
\label{sc4}

\quad In this section, 
we consider the convergence of the stochastic samplers, i.e., the backward processes,
under a ``blackbox'' assumption on score matching:
\begin{assump}
\label{assump:blackbox}
There exists $\varepsilon > 0$ such that 
$\mathbb{E}_{p(t, \cdot)}|s_{\theta}(t,X) - \nabla \log p (t, X)|^2 \le \varepsilon^2$, $t \in [0,T]$ 
for some $s_\theta(\cdot, \cdot)$.
\end{assump}
For instance, the score matching function $s_\theta(\cdot, \cdot)$
can be the output of any algorithm introduced in Section \ref{sc3}.
This condition postulates an $L^2$-bound on score matching.
As we will see in Section \ref{sc5},
there has been a body of work on the accuracy of score matching.
So these assumptions, especially Assumption \ref{assump:blackbox}
can be replaced with those score approximation bounds.

\quad The quality of the diffusion model relies on how $(f(\cdot, \cdot), g(\cdot))$ is chosen.
Since the SDEs are inherently in continuous time,
the convergence of the diffusion model as a continuous process
yields another criterion on the design of the model. 
 Ideally, we want to show that 
the stochastic sampler $Y_T$ defined by \eqref{eq:SDErevtheta},
as well as its discretization are close in distribution to the target distribution $p_{\tiny \mbox{data}}(\cdot)$.
Here we focus on two metrics:
the total variation distance and the Wasserstein-2 distance. 
While most existing convergence results were established in
the total variation distance (or KL divergence),
the convergence in the Wasserstein metric is important 
because 
it has shown to align with human judgment on image
similarity \cite{Bor19}, and the standard evaluation metric -- Fréchet inception distance (FID)
is also based on the Wasserstein distance. 

\quad For simplicity, we assume (as in Section \ref{sc22}) that
$n = d$ and $g(t,x) = g(t) I$.
In Sections \ref{sc411}--\ref{sc412},
we study the convergence of the stochastic samplers in continuous time
in the two metrics respectively.
The problem of discretization will be discussed in Section \ref{sc42}, 
with additional references given.

\subsection{Stochastic sampler -- total variation bound}
\label{sc411}

The goal is to bound $d_{TV}(\mathcal{L}(Y_T), p_{\tiny \mbox{data}}(\cdot))$,
with $Y$ defined by \eqref{eq:SDErevtheta2}.
\cite{DV21} first provided a bound of $d_{TV}(\mathcal{L}(Y_T), p_{\tiny \mbox{data}}(\cdot))$ 
for the OU process \eqref{eq:OUfg}, 
assuming an $L^{\infty}$-bound on score matching\footnote{There exists $\varepsilon > 0$ such that 
$|s_\theta(t,\cdot) - \nabla \log p(t,\cdot)|_\infty \le \varepsilon$, $t \in [0,T]$
for some $s_\theta(\cdot, \cdot)$.}.
The bound is of form:
\begin{equation}
\label{eq:expTbad}
d_{TV}(\mathcal{L}(Y_T), p_{\tiny \mbox{data}}(\cdot))  \le \underbrace{d_{TV}(p(T, \cdot), p_{\tiny \mbox{noise}}(\cdot))}_{ \le C \exp(-C'T)} + D \varepsilon \exp(D' T),
\end{equation}
for some $C, C', D, D'> 0$ depending on $\theta, \mu$ and $p_{\tiny \mbox{data}}(\cdot)$.
The disadvantage of this result is that
the score matching error grows exponentially in time $T$.
A recent breakthrough \cite{Chen23} improved this bound
to be polynomial in $T$ (and $d$).
The result is stated for the (discretized) OU process, 
but is easily extended to the general case.
\begin{theorem}[Total variation bound]
\label{thm:Chen23}
\cite{Chen23}
Let Assumption \ref{assump:blackbox} hold, 
and assume that $g(t)$ is bounded away from zero.
Then
\begin{equation}
\label{eq:mainTVbd}
d_{TV}(\mathcal{L}(Y_T), p_{\tiny \mbox{data}}(\cdot))
\le \underbrace{d_{TV}(p(T, \cdot), p_{\tiny \mbox{noise}}(\cdot))}_{\tiny \mbox{noise approx. error}}  + \underbrace{\varepsilon \sqrt{T/2}}_{\tiny \mbox{score matching error}}.
\end{equation}
\end{theorem}
\begin{proof}
We sketch the proof, which is to decompose $d_{TV}(\mathcal{L}(Y_T), p_{\tiny \mbox{data}}(\cdot))$
into the noise approximation error and the score matching error.
Denote by 
\begin{itemize}[itemsep = 3 pt]
\item
$Q_T$ the distribution of the backward process \eqref{eq:SDErevtheta2}, 
i.e., with the matched score $s_\theta(t,x)$ and $Y_0 \sim p_{\tiny \mbox{noise}}(\cdot)$;
\item
$Q'_T$ the distribution of the backward process, 
with the matched score $s_\theta(t,x)$ and $Y_0 \sim p(T, \cdot)$;
\item
$Q''_T$ the distribution of the backward process,
with the true score $\nabla \log p(t,x)$ and $Y_0 \sim p(T, \cdot)$.
\end{itemize}
By the data processing inequality, we have:
\begin{equation}
\label{eq:datapro}
\begin{aligned}
d_{TV}(\mathcal{L}(Y_T), p_{\tiny \mbox{data}}(\cdot))
& \le d_{TV}(Q_T, Q'_T) + d_{TV}(Q'_T, Q''_T) \\
& \le d_{TV}(p(T, \cdot), p_{\tiny \mbox{noise}}(\cdot)) +  d_{TV}(Q'_T, Q''_T).
\end{aligned}
\end{equation}
The key to bound the term $d_{TV}(Q'_T, Q''_T)$
is the identity:
\begin{equation}
\label{eq:KLGir}
\mbox{KL}(Q''_T, Q'_T) = \mathbb{E}_{Q''_T} \left( \log \frac{dQ''_T}{dQ'_T}\right)
= \mathbb{E}_{Q''_T} \int_0^T |s_{\theta}(T-t, Y_t) - \nabla \log  p (T-t, Y_t)|^2 dt,
\end{equation}
where the second equality follows from the Girsanov theorem. 
Further by Assumption \ref{assump:blackbox},
we get by Pinsker's inequality:
\begin{equation}
\label{eq:secondTV}
d_{TV}(Q'_T, Q''_T) \le \sqrt{\frac{1}{2} \mbox{KL}(Q''_T, Q'_T)}
\le \varepsilon \sqrt{T/2}.
\end{equation}
Combining \eqref{eq:datapro} and \eqref{eq:secondTV} yields \eqref{eq:mainTVbd}.
\end{proof}

\quad Now let's consider an important special case -- the VP SDE
that also includes the OU process.
\begin{corollary}[Total variation for VP]
\label{coro:VPSDE}
Let $(Y_t, \, 0 \le t \le T)$ be specified by \eqref{eq:sigbVP} (the backward process of VP).
Let Assumption \ref{assump:blackbox} hold, 
and assume that $\beta(t)$ is bounded away from zero,
and $\mathbb{E}_{p_{\tiny \mbox{data}}(\cdot)} |x|^2 < \infty$.
Then for $T$ sufficiently large,
\begin{equation}
\label{eq:maincoroTV}
d_{TV}(\mathcal{L}(Y_T), p_{\tiny \mbox{data}}(\cdot)) \le 
e^{-\frac{1}{2}\int_0^T \beta(s) ds} \sqrt{\mathbb{E}_{p_{\tiny \mbox{data}}(\cdot)} |x|^2/2} + \varepsilon \sqrt{T/2}.
\end{equation}
\end{corollary}
\begin{proof}
Recall from \eqref{eq:VPp} that 
for VP,
the conditional distribution of $(X_t \,|\, X_0 = x)$ is
$\mathcal{N} \left( e^{-\frac{1}{2} \int_0^t \beta(s) ds} x, (1 - e^{-\int_0^t \beta(s) ds}) I  \right)$.
By the convexity of KL divergence,
we have:
\begin{equation}
\label{eq:convexKL}
\begin{aligned}
& \mbox{KL}(p(T, \cdot), p_{\tiny \mbox{noise}}(\cdot)) \\
& \quad \le \mathbb{E}_{p_{\tiny \mbox{data}}(\cdot)}
\bigg[\underbrace{\mbox{KL}\left( \mathcal{N} \left( e^{-\frac{1}{2} \int_0^T \beta(s) ds} x, (1 - e^{-\int_0^T \beta(s) ds}) I  \right), \mathcal{N}(0, I)\right)}_{(a)}\bigg].
\end{aligned}
\end{equation}
Moreover,
\begin{equation}
\label{eq:KLGaussian}
\begin{aligned}
(a) &= e^{- \int_0^T \beta(s) ds} |x|^2 - d \left( e^{-\int_0^T \beta(s) ds} + \log (1 - e^{-\int_0^T \beta(s) ds}) \right) \\
& \le e^{- \int_0^T \beta(s) ds} |x|^2, \quad \mbox{as } T \to \infty,
\end{aligned}
\end{equation}
since $\beta(t)$ is bounded away from zero, 
and hence $\int_0^T \beta(s) ds \to \infty$ as $T \to \infty$.
Combining \eqref{eq:convexKL}, \eqref{eq:KLGaussian} with \eqref{eq:mainTVbd} yields \eqref{eq:maincoroTV}.
\end{proof}

\quad Similar polynomial bounds were also obtained by \cite{LLT22, LLT23}
for the VP SDE.
There $\chi^2$-divergence (and an identity analogous to \eqref{eq:KLGir})
is used to 
bound $d_{TV}(\mathcal{L}(Y_T), p_{\tiny \mbox{data}}(\cdot))$. 

\subsection{Stochastic sampler -- Wasserstein bound}
\label{sc412}

We make the following additional assumptions.
\begin{assump}
\label{assump:sensitivity}
The following conditions hold:
\begin{enumerate}[itemsep = 3 pt]
\item
There exists $r_f: [0,T] \to \mathbb{R}$ such that 
$(x - x') \cdot (f(t,x) - f(t,x')) \ge r_f(t) |x - x'|^2$ for all $t,x,x'$.
\item
There exists $L > 0$ such that $|s_\theta (t,x) - s_\theta (t,x')| \le L |x-x'|$ for all $t,x,x'$.
\end{enumerate}
\end{assump}

\quad The condition (1) assumes the growth of $f(t, \cdot)$,
and (2) assumes the Lipschitz property of the score matching function (as considered in \cite{KW22}). 
\begin{theorem}[Wasserstein bound]
\label{thm:sensitivity}
\cite{TZ24}
Let Assumptions \ref{assump:blackbox}, \ref{assump:sensitivity} hold.
For $h > 0$ a hyperparameter,
define
\begin{equation}
\label{eq:uv}
u(t):=\int_{T-t}^T \left(-2 r_f(s) + (2L + 2h) g^2(s) \right)ds.
\end{equation}
Then we have 
\begin{equation}
\label{eq:sensitivity}
W_2(p_{\tiny \mbox{data}}(\cdot), \mathcal{L}(Y_T)) \le \sqrt{\underbrace{W_2^2(p(T, \cdot), p_{\tiny \mbox{noise}}(\cdot)) e^{u(T)}}_{\tiny \mbox{noise approx. error}} + \underbrace{\frac{\varepsilon^2}{2h} \int_0^T g^2(t) e^{u(T)-u(T-t)} dt}_{\tiny \mbox{score matching error}}}.
\end{equation}
\end{theorem}
\begin{proof}
The idea relies on coupling.
Consider the coupled SDEs:
\begin{equation*}
\left\{ \begin{array}{lcl}
d U_t = \left(-f(T-t, U_t) + g^2(T-t) \nabla \log p(T-t, U_t) \right) dt+ g(T-t) dB_t, \\
d V_t =  \left(-f(T-t, V_t) + g^2(T-t)s_\theta(T-t, V_t) \right) dt+ g(T-t) dB_t,
\end{array}\right.
\end{equation*}
where $(U_0, V_0)$ are coupled to achieve $W_2(p(T,\cdot), p_{\tiny \mbox{noise}}(\cdot))$.
Note that
$$W_2^2(p_{\tiny \mbox{data}}(\cdot), \mathcal{L}(Y_T)) \le \mathbb{E}|U_T - V_T|^2,$$
so the goal is to bound $\mathbb{E}|U_T - V_T|^2$. 
By It\^o's formula, we get
\begin{multline*}
d|U_t - V_t|^2 = 2(U_t - V_t) \cdot (-f(T-t, U_t) + g^2(T-t) \nabla \log p(T-t, U_t) \\ + f(T-t, V_t) - g^2(T-t)s_\theta(T-t, V_t)) dt,
\end{multline*}
which implies that 
\begin{equation}
\label{eq:diff}
\begin{aligned}
\frac{d \, \mathbb{E}|U_t - V_t|^2}{dt} 
&= -2 \underbrace{\mathbb{E}((U_t - V_t) \cdot (f(T-t, U_t) - f(T-t, V_t))}_{(a)}) \\
&\quad+ 2 \underbrace{\mathbb{E}((U_t - V_t) \cdot g^2(T-t) (\nabla \log p(T-t, U_t) - s_\theta(T-t, V_t)))}_{(b)}.
\end{aligned}
\end{equation}
By Assumption \ref{assump:sensitivity} (1), we get
$(a) \ge r_f(T-t) \, \mathbb{E}|U_t - V_t|^2$.
Moreover,
\begin{equation}
\label{eq:termb}
\begin{aligned}
(b) & = g^2(T-t) \bigg( \mathbb{E}((U_t - V_t) \cdot (\nabla \log p(T-t, U_t) - s_\theta(T-t, U_t)  )) \\
& \qquad \qquad \qquad   + \mathbb{E}((U_t - V_t) \cdot (s_\theta (T-t, U_t) - s_{\theta}(T-t, V_t))) \bigg)\\
&  \le  g^2(T-t) \left( h \mathbb{E}|U_t - V_t|^2 + \frac{1}{4h} \varepsilon^2 + L\, \mathbb{E}|U_t - V_t|^2\right), 
\end{aligned}
\end{equation}
where we use Assumptions \ref{assump:blackbox} and \ref{assump:sensitivity} (2) in the last inequality.
Therefore,
\begin{equation*}
\frac{d \, \mathbb{E}|U_t - V_t|^2}{dt} \le \left(-2 r_f(T-t) + (2h + 2L) g^2(T-t) \right) \mathbb{E}|U_t -V_t|^2
+ \frac{\varepsilon^2}{2h} g^2(T-t).
\end{equation*}
Applying Gr\"onwall's inequality, we have
$\mathbb{E}|U_T - V_T|^2\leq e^{u(T)}\mathbb{E}|U_0 - V_0|^2+ \frac{\varepsilon^2}{2h} \int_0^T \sigma^2(T-t) e^{u(T)-u(t)} dt$,
which yields \eqref{eq:sensitivity}. 
\end{proof}

\quad Theorem \ref{thm:sensitivity} assumes that the score matching function is Lipschitz,
which may be too restrictive, say for a neural net.
A more natural assumption is that there exists $L > 0$ such that 
$$|\nabla \log p (t,x) - \nabla \log p (t,x')| \le L |x-x'| \mbox{ for all } t,x,x',$$
i.e., the score function (rather than the score matching function) is Lipschitz. 
In this case, the result still holds if an $L^\infty$ bound on score matching is assumed.
\cite{GNZ23} obtained a similar result by 
an $L^2$-bound on score matching
but
directly assuming that 
$\mathbb{E}|s_\theta(T-t, V_t) - \nabla \log p(T-t, V_t)|^2 \le \varepsilon^2$ for all $t$.

\quad Theorem \ref{thm:sensitivity} does not require a specific structure of 
$f(t,x)$. 
The bound \eqref{eq:sensitivity} is loose in the general case, 
e.g.,
$r_f(t)$ may be negative, and hence the bound is exponential in $T$.
Nevertheless, 
the estimates in the proof can be refined in some special cases.
Here we examine the VP model.

\quad The following result provides a bound for VP, 
assuming that $p_{\tiny \mbox{data}}(\cdot)$ is strongly log-concave.
Recall that a smooth function $h: \mathbb{R}^d \to \mathbb{R}$ is $\kappa$-strongly log-concave if
\begin{equation}
(\nabla \log h(x) - \nabla \log h(y)) \cdot (x - y) \le -\kappa |x-y|^2.
\end{equation} 
In comparison with Corollary \ref{coro:VPSDE}, 
the score matching error does not grow in $T$.
\begin{theorem}[Wasserstein bound for VP]
\label{thm:VPW2}
Let $(Y_t, \, 0 \le t \le T)$ be defined by \eqref{eq:sigbVP} (the backward process of VP).
Assume that 
$p_{\tiny \mbox{data}}(\cdot)$ is $\kappa$-strongly log-concave with $\kappa > \frac{1}{2}$,
and $\mathbb{E}_{p_{\tiny \mbox{data}(\cdot)}}|x|^2 < \infty$.
For $h < \min(\kappa, 1) - \frac{1}{2}$, we have:
\begin{equation}
\label{eq:sensitivity2}
\begin{aligned}
W^2_2(p_{\tiny \mbox{data}}(\cdot), \mathcal{L}(Y_T)) \le & e^{- \{\beta_{\min} + \beta_{\max}( 2 \min(\kappa,1) - 1 -2h)\} T} 
 \left(\mathbb{E}_{p_{\tiny \mbox{data}(\cdot)}}|x|^2 
+ o(e^{-\beta_{\min} T})d \right) \\
&  \qquad \qquad \qquad \qquad \qquad \qquad + \varepsilon^2 \frac{\beta^2_{\max}}{2h(2 \min(\kappa, 1) - 1 -2h)}.
\end{aligned}
\end{equation}
\end{theorem}

\begin{proof}
Recall that $r_f(t) = -\frac{1}{2} \beta (t)$, $g(t) = \sqrt{\beta(t)}$ and 
$p_{\tiny \mbox{noise}}(\cdot) \sim \mathcal{N}(0,I)$.
The key improvement relies on the fact \cite[Proposition 10]{GNZ23}:
If $p_{\tiny \mbox{data}}(\cdot)$ is $\kappa$-strongly log-concave,
then $p(T-t, \cdot)$ is 
$$\kappa \left(e^{-\int_0^{T-t} \beta(s) ds}+ \kappa \int_0^{T-t}e^{- \int_s^{T-t} \beta(v) dv} \beta(s) ds\right)^{-1}-\mbox{strongly log-concave}.$$
(The convolution preserves strong log-concavity.)
Thus, the term $\mathbb{E}((U_t - V_t) \cdot (\nabla \log p(T-t, U_t) - \nabla \log p(T-t, V_t)  ))$ in \eqref{eq:termb}
is bounded from above by
\begin{equation*}
- \frac{\kappa}{e^{-\int_0^{T-t} \beta(s) ds}+ \kappa \int_0^{T-t}e^{- \int_s^{T-t} \beta(v) dv} \beta(s) ds} \mathbb{E}|U_t - V_t|^2.
\end{equation*}
instead of $L \, \mathbb{E}|U_t - V_t|^2$.
Consequently, we obtain the bound \eqref{eq:sensitivity} by replacing $u(t)$ with
\begin{equation}
\label{eq:uVP}
u_{\tiny \mbox{VP}}(t): = \int_{T-t}^T \beta(s) \left(1 + 2h - \frac{2 \kappa}{e^{-\int_0^{s} \beta(v) dv}+ \kappa \int_0^{s}e^{- \int_v^{s} \beta(u) du} \beta(v) dv}\right) ds.
\end{equation}
It is easy to see that
$u_{\tiny \mbox{VP}}(T) \le \beta_{\max} (1 + 2h - 2 \min(\kappa, 1)) T$,
and 
$u_{\tiny \mbox{VP}}(T) - u_{\tiny \mbox{VP}}(T-t) \le \beta_{\max} (1 + 2h - 2 \min(\kappa, 1)) t$.
Moreover, 
\begin{equation*}
\begin{aligned}
W_2^2(p(T, \cdot), p_{\tiny \mbox{noise}}(\cdot)) 
& \le e^{-\int_0^T \beta(s) ds} \mathbb{E}_{p_{\tiny \mbox{data}(\cdot)}}|x|^2 + \frac{d}{4}e^{-2 \int_0^T \beta(s) ds} \\
& \le e^{- \beta_{\min} T} \mathbb{E}_{p_{\tiny \mbox{data}(\cdot)}}|x|^2 + \frac{d}{4} e^{-2 \beta_{\min} T}.
\end{aligned}
\end{equation*}
For $\kappa > \frac{1}{2}$,
we take $h < \min(\kappa, 1) - \frac{1}{2}$, so 
$1+ 2h - 2 \min(\kappa, 1) < 0$.
Combining \eqref{eq:sensitivity}
with the above estimates yields \eqref{eq:sensitivity2}.
\end{proof}

\quad See \cite{TZ24} for a Wasserstein bound for CVP. 
In comparison with VP, CVP (and more general CDPM) is designed to be robust to 
the score matching error and discretization (see Section \ref{sc42}),
at the cost of noise approximation error.
Practically, $T$ should not set to be too large for CVP
to control the noise approximation bias.

\subsection{Discretization}
\label{sc42}

In order to implement the diffusion model,
we need to discretize the corresponding SDE.
For the SDE \eqref{eq:SDErevtheta2},
there are several ways of discretization.
Fix $0 = t_0 < t_1 \cdots < t_N = T$ 
to be the time discretization. 
Initialize $\widehat{X}_0 \sim p_{\tiny \mbox{noise}}(\cdot)$, 
and set
\begin{equation}
\label{eq:EMdisc2}
\begin{aligned}
\widehat{X}_k = \widehat{X}_{k-1} + (-f(T - t_k, \widehat{X}_{k-1}) + 
 &g^2(T - t_{k-1}) s_\theta(T - t_{k-1}, \widehat{X}_{k-1})) (t_k - t_{k-1}) \\
& + \int_{t_{k-1}}^{t_k} g(T-t) dB_t, 
\quad \mbox{for } k = 1,\ldots, N,
\end{aligned}
\end{equation}
where $\int_{t_{k-1}}^{t_k}g(T-t) dB_t$
is sampled as $\mathcal{N}(0, \int_{t_{k-1}}^{t_k}g^2(T-t) dt)$,
or the Euler-Maruyama scheme:
\begin{equation}
\label{eq:EMdisc}
\begin{aligned}
\widehat{X}_k = \widehat{X}_{k-1} + (-f(T - t_k, &\widehat{X}_{k-1}) + 
 g^2(T - t_{k-1}) s_\theta(T - t_{k-1}, \widehat{X}_{k-1})) (t_k - t_{k-1}) \\
& + g(T-t_{k-1}) (B_{t_k} - B_{t_{k-1}}), 
\quad \mbox{for } k = 1,\ldots, N.
\end{aligned}
\end{equation}
In practice,
the predictor-corrector (PC) sampler \cite{Song20} is frequently used,
where the scheme \eqref{eq:EMdisc2} or \eqref{eq:EMdisc} serves as
the predictor,
followed by a score-based corrector.
More precisely, for each $k$, 
let $\widehat{X}^{(0)}_k$ be the output of the predictor 
\eqref{eq:EMdisc2} or \eqref{eq:EMdisc},
and fix some $\epsilon_k > 0$.
The corrector $\widehat{X}_k = \widehat{X}_k^{(M)}$
is given by
\begin{equation}
\label{eq:corrector}
\widehat{X}^{(\ell+1)}_k = \widehat{X}^{(\ell)}_k + \epsilon_k s_\theta(T-t_k, \widehat{X}^{(\ell)}_k)
+ \sqrt{2}(B'_{(\ell + 1) \epsilon_k} - B'_{\ell \epsilon_k}),
\quad \mbox{for } \ell = 0, \ldots, M-1,
\end{equation}
where $M$ is the number of steps in the corrector step.
The hyperparameters $\epsilon_k$ are 
taken to be decreasing in $k$.
Essentially,
\eqref{eq:corrector} is the discrete scheme of 
the Langevin equation 
$dX_t = \nabla \log p (T-t_k,X_t) + \sqrt{2} B'_t$
that converges to $p(T-t_k, \cdot)$,
but with the score function replaced with the estimated $s_\theta(T-t_k, \cdot)$.
Since the scheme \eqref{eq:EMdisc2} or \eqref{eq:EMdisc} 
is the predictor of the PC sampler, 
it is also called the predictor-only sampler.

\quad Again as seen in Section \ref{sc22},
most popular diffusion models have a common structure: 
$f(t,x) = f(t)x$, with $f(t) \in \mathbb{R}$.
In this case, the exponential integrator (EI) \cite{ZC23} can be applied to solve 
the SDE \eqref{eq:SDErevtheta2}, or the ODE \eqref{eq:ODEreverse2}.
Letting $\Psi(s,t):=e^{-\int_s^t f(T-u)du}$, we have for $0 \le s < t \le T$,
\begin{equation}
Y_t = \Psi(s,t) Y_s + \int_s^t \Psi(u,t) g^2(T-u) s_{\theta}(T-u, Y_u) du + \int_s^t \Psi(u,t) g(T-u) dB_u, \tag{25'}
\end{equation}
or
\begin{equation}
Y_t = \Psi(s,t) Y_s + \frac{1}{2}\int_s^t \Psi(u,t) g^2(T-u) s_{\theta}(T-u, Y_u) du. \tag{45'}
\end{equation}
So the discretization is:
\begin{equation}
\label{eq:EISDE}
\begin{aligned}
\widehat{X}_k = \Psi(t_{k-1}, t_k) \widehat{X}_{k-1} & + \left(\int_{t_{k-1}}^{t_k} \Psi(t_{k-1}, u) g^2(T -u) du\right) s_{\theta}(T-t_{k-1}, \widehat{X}_{k-1}) \\
&+ \int_{t_{k-1}}^{t_k}\Psi(t_{k-1}, u) g(T - u) dB_u, \quad \mbox{for } k = 1, \ldots, N.
\end{aligned}
\end{equation}
or
\begin{equation}
\label{eq:EIODE}
\begin{aligned}
\widehat{X}_k = \Psi(t_{k-1}, t_k) \widehat{X}_{k-1} + \frac{1}{2} \left(\int_{t_{k-1}}^{t_k} \Psi(t_{k-1}, u) g^2(T -u) du \right) & s_{\theta}(T-t_{k-1}, \widehat{X}_{k-1}), \\
& \mbox{for } k = 1,\ldots, N.
\end{aligned}
\end{equation}
Note that the EI solvers \eqref{eq:EISDE}--\eqref{eq:EIODE} for VP
correspond to DDPM and DDIM, respectively.
See also \cite{WCW24, ZSC23} for higher order EI solvers.

\quad Classical SDE theory \cite{KP92}
indicates that the discretization error is of 
order $C(T) \delta$, with $C(T)$ exponential in $T$.
A natural question is whether there are better bounds 
for the SDEs arising from the diffusion models in Section \ref{sc22}.
The general way to control the discretization error is to
prove a one-step estimate, and then unfold the recursion.
The analysis is mostly tedious,
and we do not plan to expand the details here.
The table below summarizes known results 
on the discretization error of various diffusion models.

\vskip 10 pt

\begin{center}
    \begin{tabular}{| c | c | c | c| c| c|}
    \hline
     Ref. & Metric & Model & Sampling & Step size & Discretization error  \\ \hline
       \cite{Chen23} & TV  &  OU & Euler & $\delta$ & $\sqrt{T  d} \sqrt{\delta}$ \\ \hline
       \cite{LLT22, LLT23} & TV & OU & PC & varying & $\mathcal{O}_T(1) \sqrt{d} \sqrt{\delta}$ \\  \hline
       \cite{WCW24}& $TV $ & OU & EI & $\delta$ & $\sqrt{Td^7} \sqrt{\delta^3}$ \\ \hline
       \cite{GNZ23}& $W_2$ & VE & Euler & $\delta$ & $e^{\mathcal{O}(T)} \sqrt{d} \delta $ \\ \hline
       & $W_2$ & VP& Euler & $\delta$  & $e^{\mathcal{O}(T^3)} \sqrt{d} \delta$ \\ \hline
       \cite{TZ24} & $W_2$ & CDPM & Euler & $\delta$ &  $\mathcal{O}_T(1) \sqrt{d} \sqrt{\delta} $ \\ \hline
    \end{tabular}\\
    \vskip 8 pt
    {TABLE $2$. Discretization errors.}
\end{center}

\quad There has been work 
on the convergence of diffusion models in a compact domain,
or satisfying the manifold hypothesis \cite{DeB22, pid22}.
In this case,
the Wasserstein distance is simply 
bounded by the total variation distance.
\cite{Chen23, LLT22, LLT23} also considered 
bounded target distributions,
which should not be confused with diffusion models in a compact domain
because the SDEs can live in an unbounded space.
An early stopping argument is used there to guarantee the convergence
of diffusion models in the $W_2$ metric.
See also \cite{Ben23a}
for finer results on the convergence of the OU process,
and \cite{LC24, LH24, LW23} that worked directly on the convergence of (discrete) DDPM.

\quad The theory of the ODE sampling of diffusion models is,
however,
still sparse.
\cite{Chen23b, HHL25} studied
the total variation convergence of the 
ODE sampler for the OU process,
and \cite{GZ24} 
established Wasserstein bounds 
for the ODE sampler of VE and VP.
See also \cite{Ben23b, Chen23c, LH24, LW23, LY24}
for closely related results 
on the deterministic sampler of diffusion models.

\quad Recently, there has been a surge of interest in understanding the convergence of diffusion models,
where the target distribution is concentrated on low-dimensional manifolds in a higher-dimensional space.
This phenomenon is called ``adapting to unknown low-dimensional structures'' \cite{LYl24}.
Typically, the dependence of the discretization error in $d$ 
is replaced with some intrinsic dimension $k \ll d$ (plus a polynomial $\log d$ correction).
See also \cite{HWC24, LHC25, PAD24, TY25} for further development in this direction.

\section{Further results on score matching}
\label{sc5}

\quad In the previous sections,
we have seen that 
the success of diffusion models 
relies largely on the quality of score matching.
This section reviews further results 
on  score matching
in the context of diffusion models.
Section \ref{sc51}
considers
statistical efficiency of score matching,
and 
Section \ref{sc52} 
presents approximations by neural nets.
In Section \ref{sc53},
we introduce 
the idea of diffusion model alignment 
by reinforcement learning (RL).

\subsection{Statistical efficiency of score matching}
\label{sc51}

In this subsection,
we focus on the statistical aspect 
of the score matching problem \eqref{eq:equivscoremat}.
Assume that
\begin{equation}
\label{eq:assumpsm}
s_\theta(t,x) = \nabla \log p_\theta(t,x), \quad \mbox{for all } t,x,
\end{equation}
where $p_\theta(t,\cdot)$ is a family of probability distributions parametrized by $\theta$.
In other words,
$s_\theta(\cdot, \cdot)$ is the score function of a parametric family of probability distributions 
instead of an arbitrary neural net.
This assumption is restrictive (and may seem to not be realistic),
but it serves as a first step to understand the statistical efficiency of score matching in the context of diffusion models.

\quad Under the assumption \eqref{eq:assumpsm}, the score matching problem \eqref{eq:equivscoremat} becomes
\begin{equation}
\label{eq:scoremat51}
\min_\theta \mathcal{J}_{\text{ESM}}(\theta):=\mathbb{E}_{p(t, \cdot)} |\nabla \log p_\theta(t,x) - \nabla \log p(t,X)|^2,
\end{equation}
and the corresponding implicit score matching problem is:
\begin{equation}
\label{eq:equivscoremat51}
\begin{aligned}
\min_\theta \mathcal{J}_{\text{ISM}}(\theta):&= \mathbb{E}_{p(t, \cdot)} \left[ |\nabla \log p_\theta(t,X)|^2  + 2 \, \nabla \cdot \nabla \log p_\theta (t,X)\right] \\
& = \mathbb{E}_{p(t, \cdot)} \left[ |\nabla \log p_\theta(t,X)|^2  + 2 \tr (\nabla^2 \log p_\theta (t,X))\right].
\end{aligned}
\end{equation}
Let $\widehat{\theta}_n$ be minimum point of the empirical version
of \eqref{eq:scoremat51} or \eqref{eq:equivscoremat51} from $n$ samples $X_1, \ldots, X_n \sim p(t, \cdot)$:
\begin{equation}
\widehat{\theta}_n:= \argmin_\theta \frac{1}{n} \sum_{i = 1}^n \left( |\nabla \log p_\theta(t,X_i)|^2  + 2 \tr (\nabla^2 \log p_\theta (t,X_i) )\right).
\end{equation}
The question of interest is whether
the estimated $p_{\widehat{\theta}_n}(t, \cdot)$ approximates well 
the probability distribution $p(t, \cdot)$. 
The answer is hinted in the following theorem.
\begin{theorem}[Consistency and asymptotic normality]
\label{thm:cansm}
\cite{Hyv07, Koe22}
Fix $t \in [0,T]$.
\begin{enumerate}[itemsep = 3 pt]
\item
Assume that $p(t, \cdot)$ has full support, and 
$p(t, \cdot) = p_{\theta_*}(t, \cdot)$ for some $\theta_*$
and 
$p (t, \cdot) \ne p_\theta (t, \cdot)$ for any $\theta \ne \theta_*$.
Then $\mathcal{J}_{\text{ESM}}(\theta) = 0$ if and only if $\theta = \theta_*$.
\item
Assume further that $p_\theta(t, \cdot)$ belongs to the exponential family,
i.e.
$p_\theta(t, \cdot) \propto \exp(\theta \cdot F(\cdot))$ 
for a suitably nice $F$.
Then 
\begin{equation}
\label{eq:CVSM}
\sqrt{n} (\widehat{\theta}_n - \theta_*) \to \mathcal{N}(0, \Gamma_{\tiny \mbox{SM}}),
\end{equation}
with
\begin{multline*}
\Gamma_{\tiny \mbox{SM}}:= \mathbb{E}_{p_{\theta_*}(t, \cdot)} \bigg[(\nabla F(X) \nabla F(X)^\top)^{-1} \Sigma_{\nabla F(X) \nabla F(X)^\top \theta_* + \Delta F(X)} \mathbb{E}(\nabla F(X) \nabla F(X)^\top)^{-1}\bigg],
\end{multline*}
where 
$\Sigma_{\bullet}$ is the Fisher information matrix,
and the Laplacian $\Delta F(X)$ applies coordinate-wise to $F(X)$.
\end{enumerate}
\end{theorem}

\begin{proof}
(1) If $\mathcal{J}_{\text{ESM}}(\theta) = 0$,
 then $\nabla \log p_\theta (t, \cdot) = \nabla \log p (t, \cdot)$.
So $p_\theta (t, \cdot) = p(t, \cdot) + C$, 
and $C = 0$ because $p(t, \cdot)$ and $p_\theta(t, \dot)$ are probability distributions.
By the hypothesis, we have $\theta = \theta_*$.

(2) Recall that $p(t, \cdot) = p_{\theta_*}(t, \cdot)$.
The key to the analysis relies on the fact that
\begin{equation}
\theta_* = - \mathbb{E}(\nabla F(X) \nabla F(X)^\top)^{-1} \mathbb{E} \Delta F(X).
\end{equation}
Denote by $\widehat{\mathbb{E}}_n$ the empirical expectation of $n$ i.i.d. random variables distributed by $p(t, \cdot)$.
Write
$\widehat{\mathbb{E}}_n(\nabla F(X) \nabla F(X)^\top) = \mathbb{E}(\nabla F(X) \nabla F(X)^\top) + \delta_{1,n}/\sqrt{n}$
and 
$\widehat{\mathbb{E}}_n \Delta F(X) = \mathbb{E} F(X) + \delta_2/\sqrt{n}$.
Note that
\begin{equation*}
\begin{aligned}
\widehat{\theta}
& = - \widehat{\mathbb{E}}_n(\nabla F(X) \nabla F(X)^\top)^{-1} \widehat{\mathbb{E}}_n\Delta F(X) \\
& = - \left( \underbrace{\mathbb{E}(\nabla F(X) \nabla F(X)^\top)^{-1} \widehat{\mathbb{E}}_n(\nabla F(X) \nabla F(X)^\top)}_{I + \mathbb{E}(\nabla F(X) \nabla F(X)^\top)^{-1} \delta_{1,n}/\sqrt{n}}  \right)^{-1} \\
& \qquad \qquad \qquad \qquad \qquad \qquad \qquad \qquad \underbrace{\mathbb{E}(\nabla F(X) \nabla F(X)^\top)^{-1} \widehat{\mathbb{E}}_n\Delta F(X)}_{-\theta_* + \mathbb{E}(\nabla F(X) \nabla F(X)^\top)^{-1} \delta_{2,n}/\sqrt{n}} \\
& = \left( -I + \mathbb{E}(\nabla F(X) \nabla F(X)^\top)^{-1} \frac{\delta_{1,n}}{\sqrt{n}} + \mathcal{O}\bigg(\frac{1}{n}\bigg)\right) \\
& \qquad \qquad \qquad \qquad \qquad \qquad \qquad \qquad \left(-\theta_* + \mathbb{E}(\nabla F(X) \nabla F(X)^\top)^{-1} \frac{\delta_{2,n}}{\sqrt{n}} \right) \\
&= \theta_* - \mathbb{E}(\nabla F(X) \nabla F(X)^\top)^{-1} \frac{\theta_* \delta_{1,n} + \delta_{2,n}}{\sqrt{n}} + \mathcal{O}(1/\sqrt{n}),
\end{aligned}
\end{equation*}
where we use Slutsky's theorem in the third equation.
Moreover, 
\begin{multline*}
\frac{\theta_* \delta_{1,n} + \delta_{2,n}}{\sqrt{n}} \\ = \mathbb{E}(\nabla F(X) \nabla F(X)^\top \theta_* + \Delta F(X)) - \widehat{\mathbb{E}}_n(\nabla F(X) \nabla F(X)^\top \theta_* + \Delta F(X)),
\end{multline*}
which converges in distribution to $\mathcal{N}(0, \Sigma_{\nabla F(X) \nabla F(X)^\top \theta_* + \Delta F(X)})$.
Combining the above estimates yields \eqref{eq:CVSM}.
\end{proof}

\quad Under the assumptions in Theorem \ref{thm:cansm},
it is clear that 
Assumption \ref{assump:blackbox} ($L^2$-bound on score matching) holds
if 
\begin{equation*}
\mathbb{E}_{p(t, \cdot)}\left(\sup_\theta |\partial_\theta p_\theta(t, X)|^{2+\epsilon}\right) < \infty,
\end{equation*}
for any $\epsilon > 0$.
Another interesting result from \cite{Koe22} shows that
under an isoperimetric assumption on $p(t, \cdot)$,
the statistical efficiency of score matching is close to the maximum likelihood estimate
by comparing $\Gamma_{\tiny \mbox{SM}}$ with $\Gamma_{\tiny \mbox{MLE}} = \Sigma_{F(X)}^{-1}$.
Similar results were obtained by \cite{PR23} for a more general class of distributions.

\subsection{Score matching by neural nets}
\label{sc52}

Here we consider score matching for VP.
Recall $\gamma(t):= \exp\left(-2 \int_0^t \beta(s) ds \right)$.
By Tweedie's formula \cite{Efron11, Robbins56},
the score function can be written as:
\begin{equation}
\label{eq:GcscoreVP}
\nabla \log p(t, x)= \frac{\gamma(t)^{\frac{1}{4}}}{1 - \sqrt{\gamma(t)}} \mathbb{E}(X_0 \,|\, X_t = x) - \frac{1}{1 - \sqrt{\gamma(t)}} x,
\end{equation}
Inspired by the special structure \eqref{eq:GcscoreVP},
it is natural to take the function approximation:
\begin{equation}
\label{eq:Unet}
s_\theta(t,x) =\frac{\gamma(t)^{\frac{1}{4}}}{1 - \sqrt{\gamma(t)}} f_\theta(t,x) - \frac{1}{1 - \sqrt{\gamma(t)}} x,
\end{equation}
with $f_\theta(t,x)$ a suitable neural net (e.g. ReLU net).
As pointed out by \cite{CH23},
the specific form \eqref{eq:Unet} has close affinity to the U-net,
where the term $-\frac{1}{1 - \sqrt{\gamma(t)}} x$ corresponds to a shortcut connection.

\quad The following theorem provides a template for a class of results
concerning the approximation efficiency of score matching.
To simplify the presentation,
we take the weight $\lambda(t) = 1$.
\begin{theorem}
\cite{CH23, HR24}
Let $\widehat{s}_{\theta}(\cdot, \cdot)$ solve
one of the empirical score matching problems (in Section \ref{sc3}),
where $\mathbb{E}$ is replaced with $\widehat{\mathbb{E}}$ under finite samples.
Let $\widehat{s}_{\theta, k}$ be the output at iteration $k$ of some algorithm (e.g. SGD).
For $\widehat{s} = \widehat{s}_{\theta}$ or $\widehat{s}_{\theta, k}$,
we have: 
\begin{equation}
\label{eq:decompCH}
\begin{aligned}
& \int_{0}^T \mathbb{E}_{p(t, \cdot)}|\widehat{s}(t, \cdot) - \nabla \log p (t, \cdot)|^2 dt \\
& \qquad \le \int_{0}^T  \mathbb{E}_{p(t, \cdot)}|\widehat{s}(t, \cdot) - \nabla \log p (t, \cdot)|^2 
- \widehat{\mathbb{E}}_{p(t, \cdot)}|\widehat{s}(t, \cdot) - \nabla \log p (t, \cdot)|^2  dt \\
& \qquad \qquad +\inf_{\theta} \int_{0}^T \widehat{\mathbb{E}}_{p(t, \cdot)}| s_\theta(t, \cdot) - \nabla \log p (t, \cdot)|^2 dt,
\end{aligned}
\end{equation}
where the first term on the right side of \eqref{eq:decompCH} is the sample error of
score matching,
and the second term is the approximation error.
\end{theorem}

\quad In \cite{CH23}, $f_\theta(t,x)$ is taken to be a multilayer ReLU with linear structure.
But no specific algorithm of stochastic optimization is used 
(corresponding to $\widehat{s} = \widehat{s}_\theta$).
\cite{HR24} considered a two-layer ReLU net, 
and gradient descent is used to update $\widehat{s} = \widehat{s}_{\theta,k}$.
In particular, they used a neural tangent kernel as a universal proxy to the score,
and 
showed that score matching (using a two-layer ReLU net) is as easy as kernel regression.
\cite{WHT24} considered a feedforward net, also with gradient descent to update $\widehat{s}_{\theta,k}$.
In general, the bound \eqref{eq:decompCH} depends,
in a complicated manner,
on the data distribution, 
the net structure,
the sample size,
and the number of iterations in the optimization algorithm.

\subsection{Score matching by reinforcement learning}
\label{sc53}

In this last subsection, 
we consider score matching by RL,
which was initiated by \cite{Black23, fan2023optimizing, Lee23}.
The idea is to treat the score as the action in RL,
where the reward measures how well the sample $Y_T$ 
aligns with human preferences (diffusion model alignment).
Since the diffusion model is specified by the SDE,
the recently developed continuous RL \cite{Jia22, Jia222, WZZ20, ZTY23} provides 
a convenient framework to formulate score matching by RL.

\quad Let's recall the setup of continuous RL.
The state dynamics $(X^a_t, \, t \ge 0)$ is governed by the SDE:
\begin{equation}
\label{SDE_dynamic}
d X_t^a=b\left(t, X_t^a, a_t\right)dt+\sigma\left(t\right) d B_t,\quad X^a_0\sim p_{\tiny \mbox{init}}(\cdot),
\end{equation}
where $b:[0,T]\times\mathbb{R}^d \times \mathcal{A} \to \mathbb{R}^d$,
$\sigma:[0,T]\to \mathbb{R}_+$,
and the action $a_s \in \mathcal{A}$ is generated from a distribution $\pi\left(\cdot \,|\, t, X^a_t\right)$ by external randomization.
Define the (state) value function given the feedback policy $\{\pi(\cdot \,|\, t, x): x \in \mathbb{R}^d\}$ by
\begin{equation*}
\label{Value function Definition}
\begin{aligned}
V( t, x ; \pi) : = \mathbb{E} \left[\int_t^{T} r\left(s, X_s^\pi, a_s^\pi\right) \mathrm{d} s +h(X_T^{\pi}) \, \bigg| \, X_0^\pi = x \right],
\end{aligned}
\end{equation*}
where $r:\mathbb{R}_+ \times \mathbb{R}^d \times \mathcal{A} \to \mathbb{R}$ and $h: \mathbb{R}^d \to \mathbb{R}$ are the running and terminal rewards, respectively.
The performance metric is 
$\eta(\pi):= \int_{\mathbb{R}^d} V(0, x; \pi) p_{\tiny \mbox{init}}(dx)$.
The main task of continuous RL is to approximate
$\max_\pi \eta(\pi)$ by constructing a sequence of policies $\pi_k$, $k = 1,2,\ldots$ recursively such that $\eta(\pi_k)$ is non-decreasing.

\quad Now we introduce the main idea from \cite{ZZT24, ZCZ25} (see also \cite{GZZ24}).
Comparing the backward process \eqref{eq:SDErevexp}
with the continuous RL \eqref{SDE_dynamic},
it is natural to set $p_{\tiny \mbox{init}}(\cdot) = p_{\tiny \mbox{noise}}(\cdot)$,
\begin{equation}
b(t,x,a)= -f(T-t, x) + g^2(T-t) a \quad \mbox{and} \quad \sigma(t) = g(T-t).
\end{equation}
Here the action $a$ has a clear meaning -- 
it is at the place of the score function. 
So we can deploy the continuous RL machinery to ``fine-tune'' the score.
At time $t$, we use a Gaussian exploration 
$a^\theta_t \sim \pi_{\theta}(\cdot\mid t,Y^\theta_t) = N(\mu_{\theta}(t,Y^\theta_t),\sigma_t^2I)$, 
where $\sigma_t$ is a predetermined exploration level for each $t$.
Under the Gaussian exploration, the RL dynamics is:
\begin{equation}
d Y_t^{\theta}=\left[-f(T-t,Y_t^{\theta}) + g^2(T-t) \mu_{\theta}(t,Y_t^{\theta})\right] dt+g(T-t) dB_t,\,\, Y^{\theta}_0\sim p_{\tiny \mbox{noise}}(\cdot).
\end{equation}
Denote by $p_\theta(t, \cdot)$ the probability distribution of $Y^\theta_t$.
In particular, the score matching parametrization $\mu_{\theta}(t, Y^\theta_t) = s_{\theta}(T-t,Y_t^{\theta})$
recovers the backward process \eqref{eq:SDErevtheta2}.

\quad In order to prevent the model from overfitting to the reward or catastrophic forgetting, 
we add a penalty term to ensure that
the pretrained model $s_{\tiny \mbox{pre}}(\cdot, \cdot)$ is fully exploited.
The problem that we want to solve is:
\begin{equation}
\label{objective with regularization}
\max_\theta \mathbb{E}\left[R(Y^{\theta}_T)-\beta \operatorname{KL}\left(p_{\theta}(T,\cdot), p_{\theta_{pre}}(T,\cdot)\right)\right],
\end{equation}
where $\beta>0$ is a (given) penalty cost,
$R: \mathbb{R}^d \to \mathbb{R}$ is a reward to evaluate the sample $Y^\theta_T$.
It is easy to see that the objective function \eqref{objective with regularization} is equivalent to:
\begin{align*}
\mathbb{E} \int_0^{T}\underbrace{-\frac{\beta}{2} g^2(T-t)|a_t^{\theta}-\mu_{\theta_{pre}}(t, Y^\theta_t)|^2}_{r (t,Y^\theta_t,a_t^{\theta})} \mathrm{d} t +\mathbb{E}\underbrace{R(Y^{\theta}_T)}_{h(Y^{\theta}_T)},
\end{align*}
We also define the corresponding value function:
\begin{align}
\label{continuous-time value function with regularization}
V(t,x; \pi_\theta)= 
\mathbb{E} \bigg[\int_t^{T}-\frac{\beta}{2} g^2(T-t)|a_t^{\theta}-\mu_{\theta_{pre}}(t, Y^\theta_t)|^2\mathrm{d} t +R(Y^{\theta}_T) \,\bigg|\, Y^{\theta}_t=x\bigg].
\end{align}


\quad The following theorem gives the policy gradient with respect to the performance metric $\eta$. 
\begin{theorem}[Policy gradient for score matching]
\cite{GZZ24, ZCZ25}
\label{thm:PG formula simplified}
The policy gradient of a policy $\pi^{\theta}$ parameterized by $\theta$ is:
\begin{equation}
\nabla_{\theta} V(t,x; \pi_\theta)= \mathbb{E}\left[\int _ {t} ^ { T }\nabla_{\theta} \log \pi_{\theta} ( a _ {s}^{\theta} | s , Y _ {s}^{\theta} ) q(t, Y_s^{\theta}, a_s^{\theta} ; \pi_\theta)ds \, \bigg| \, Y^\theta_t = x\right],
\end{equation}
where 
$q(t, x, a ; \pi_\theta)=\frac{\partial V}{\partial t}\left(t, x ; \pi_\theta\right)+ g^2(T-t) \frac{\partial V}{\partial x}\left(t,x ; \pi_\theta \right) a + r(t,x,a)$ is the advantage function in continuous RL.
\end{theorem}

\begin{proof}
It follows from \cite[Theorem 5]{Jia22} that for all $t,x$,
\begin{equation}
\label{eqn:Jia&Zhou PG}
\begin{aligned}
& \nabla_\theta V(t, x ; \pi_\theta) \\
& \qquad =  \mathbb{E}\bigg[\int_t ^ { T }\nabla_\theta \log \pi_{ \theta} ( a _ {s} ^ {\theta} | s , Y _ { s } ^ { \theta } ) \bigg(d V(s, Y_s^{\theta} ; \pi_\theta) +R(Y_s^{\theta}) d s\bigg)  \, \bigg| \, Y_t^{\theta}=x\bigg].
\end{aligned}
\end{equation}
By applying It\^o's formula to $V(t, Y^\theta_t; \pi_\theta)$, we have:
\begin{equation}
\label{eq:ItoV}
\begin{aligned}
& d V(t, Y^\theta_t; \pi_{\theta}) \\
& = 
\bigg[ \partial_t V(t, Y^\theta_t; \pi_{\theta}) + \frac{1}{2} g^2(T-t) \Delta V(t, Y^\theta_t; \pi_{\theta}) - f(T-t, Y^\theta_t) \frac{\partial V}{\partial x}(t,Y^\theta_t, \pi_\theta) \\
& \qquad  + g^2(T-t)  \frac{\partial V}{\partial x}(t,Y^\theta_t, \pi_\theta)a_t^\theta  + r(t, Y^\theta_t, a^\theta_t) \bigg] dt 
+  \frac{\partial V}{\partial x}(t,Y^\theta_t; \pi_\theta)g(T-t) dB_t.
\end{aligned}
\end{equation}
Injecting \eqref{eq:ItoV} into \eqref{eqn:Jia&Zhou PG},
and noting 
$g^2(T-t)  \frac{\partial V}{\partial x}(t,Y^\theta_t, \pi_\theta)a_t^\theta  + r(t, Y^\theta_t, a^\theta_t)$
are the only terms involving $a_t^{\theta}$ yields the desired result.
\end{proof}

\quad Relying on Theorem \ref{thm:PG formula simplified}, the machinery in \cite{ZTY23}
provides various algorithms (e.g., proximal policy optimization (PPO) \cite{schulman2017}) to implement score fine-tuning.

\section{Conclusion and future directions}
\label{sc8}

\quad In this section, we collect a few open problems
which we hope to trigger future research.

\begin{enumerate}[itemsep = 3 pt]
\item
The theory for 
the ODE samplers, and especially the consistency model is sparse.
\cite{LHW24} studied the convergence rate of a modified CD.
It would be useful to establish provable/sharp guarantees for
the original CD. 
\item
There has also been work on diffusion models in constraint domains
\cite{DS22, FK23, LE23, RDD22}.
Establishing the convergence result in each particular case is of independent interest.
\item
Approximation efficiency for score matching were studied
in simple scenarios,
e.g., exponential families, two-layer neural nets. 
It would be interesting to see whether 
one can go beyond these cases 
to match the practice.
\item 
Diffusion guidance \cite{Dh21, HS21} and fine-tuning \cite{UZ24, UZ242}
are used to generate data that is customized for specific downstream tasks.
Some recent work \cite{Guo24, Tang24, TX25, WC24}
explored the theoretical aspects of these approaches,
which calls for further development.
\item
Diffusion model alignment is a promising subject
in the context of RL from human feedback (RLHF).
Rudiments are provided in Section \ref{sc53},
and it remains to build an entire framework, 
with provable guarantees.
\end{enumerate}

\bigskip
{\bf Acknowledgement}: 
We thank Yuxin Chen, Sinho Chewi, Xuefeng Gao and Yuting Wei for helpful discussions.
Hanyang Zhao is supported by NSF grant DMS-2206038,
and InnoHK Initiative, The
Government of the HKSAR and the AIFT Lab.
Wenpin Tang is supported by NSF grants DMS-2113779 and
DMS-2206038, 
the Columbia Innovation Hub grant, 
and the Tang Family Assistant Professorship.

\bibliographystyle{abbrv}
\bibliography{unique}
\end{document}